\documentclass[twoside]{article}

\usepackage[accepted]{aistats2020}
\usepackage[numbers,sort&compress]{natbib}

\usepackage[utf8]{inputenc} 
\usepackage[T1]{fontenc}    
\usepackage[hidelinks]{hyperref}       
\usepackage{url}            
\usepackage{booktabs}       
\usepackage{amsfonts}       
\usepackage{nicefrac}       
\usepackage{microtype}      
\usepackage{graphicx}
\usepackage{subfig}
\graphicspath{{figures/} }
\usepackage{xspace}

\usepackage{amsmath}
\usepackage{amsfonts}
\usepackage{amssymb}
\usepackage{amsthm}
\usepackage{bm}
\usepackage{bbm}
\usepackage{mathtools}
\usepackage{enumitem}
\usepackage{thmtools,thm-restate}
\usepackage{algorithm}
\usepackage{algorithmic}
\usepackage{natbib}
\usepackage[capitalise]{cleveref}
\usepackage{color}
\usepackage[dvipsnames]{xcolor}
\usepackage{graphicx}
\usepackage{comment}

\theoremstyle{plain}
\newtheorem{lemma}{Lemma}
\newtheorem{theorem}{Theorem}
\newtheorem{proposition}{Proposition}

\theoremstyle{definition}

\newcommand{\red}[1]{\textcolor{red}{#1}}
\newcommand{\blue}[1]{\textcolor{blue}{#1}}

\def\AA{\mathcal{A}}\def\CC{\mathcal{C}}

\def\LL{\mathcal{L}}
\def\MM{\mathcal{M}}\def\NN{\mathcal{N}}
\def\PP{\mathcal{P}}
\def\SS{\mathcal{S}}
\def\VV{\mathcal{V}}\def\XX{\mathcal{X}}
\def\YY{\mathcal{Y}}\def\ZZ{\mathcal{Z}}

\def\Eb{\mathbf{E}}
\def\Ib{\mathbf{I}}
\def\Lb{\mathbf{L}}

\def\Pb{\mathbf{P}}

\def\ab{\mathbf{a}}\def\bb{\mathbf{b}}
\def\db{\mathbf{d}}\def\fb{\mathbf{f}}
\def\gb{\mathbf{g}}

\def\pb{\mathbf{p}}\def\qb{\mathbf{q}}\def\rb{\mathbf{r}}
\def\ub{\mathbf{u}}
\def\vb{\mathbf{v}}\def\wb{\mathbf{w}}\def\xb{\mathbf{x}}

\def\Ebb{\mathbb{E}}

\def\Rbb{\mathbb{R}}

\def\Vbb{\mathbb{V}}

\def\vv{\boldsymbol{v}}

\def\R{\Rbb}

\def\oneb{{\mathbf1}}

\def\zerob{{\mathbf0}}

\def\t{\top}
\def\*{\star}

\newcommand{\norm}[1]{ \| #1 \|  }

\newcommand{\lr}[2]{ \left\langle #1, #2 \right\rangle}
\DeclareMathOperator*{\argmin}{arg\,min}
\DeclareMathOperator*{\argmax}{arg\,max}

\def\regret{\textrm{Regret}}

\newcommand{\E}{\Ebb}

\newtheorem*{theorem*}{Theorem}
\newtheorem*{lemma*}{Lemma}
\newtheorem*{proposition*}{Proposition}

\def\mub{{\bm{\mu}}}
\def\pib{{\bm{\pi}}}

\def\Phib{{\bm{\Phi}}}
\def\Psib{{\bm{\Psi}}}
\def\phib{{\bm{\phi}}}
\def\psib{{\bm{\psi}}}
\def\thetab{{\bm{\theta}}}
\def\deltab{{\bm{\delta}}}

\def\EP{\textrm{EP}}

\newcommand{\since}[1]{\text{($\because$ #1)}}

\renewcommand{\red}[1]{\leavevmode{\color{black}{#1}}}
\renewcommand{\blue}[1]{\leavevmode{\color{black}{#1}}}

\begin{document}

%

%

\twocolumn[

\runningtitle{A Reduction from Reinforcement Learning  to Online Learning}

\aistatstitle{
	A Reduction from Reinforcement Learning to \\No-Regret Online Learning
}



\aistatsauthor{ Ching-An Cheng \And Remi Tachet des Combes \And  Byron Boots \And Geoff Gordon}

\aistatsaddress{ Georgia Tech \And  Microsoft Research Montreal \And UW \And Microsoft Research Montreal } ]

\begin{abstract}
%
We present a reduction from reinforcement learning (RL) to no-regret online learning based on the saddle-point formulation of RL, by which \emph{any} online algorithm with sublinear regret can generate policies with provable performance guarantees.
This new perspective decouples the RL problem into two parts: regret minimization and function approximation. 
The first part admits a standard online-learning analysis, and the second part can be quantified independently of the learning algorithm. 
Therefore, the proposed reduction can be used as a tool to systematically design new RL algorithms. 
We demonstrate this idea by devising a simple RL algorithm based on mirror descent and the generative-model oracle.
For any $\gamma$-discounted tabular RL problem, with probability at least $1-\delta$, it learns an $\epsilon$-optimal policy using at most
$\tilde{O}\left(\frac{|\SS||\AA|\log(\frac{1}{\delta})}{(1-\gamma)^4\epsilon^2}\right)$ samples.
Furthermore, this algorithm admits a direct extension to linearly parameterized function approximators for large-scale applications, \red{with computation and sample complexities independent of $|\SS|$,$|\AA|$, though at the cost of potential approximation bias.}
%
 
\end{abstract}

\section{INTRODUCTION}


Reinforcement learning (RL) is a fundamental problem for sequential decision making in unknown environments. 
One of its core difficulties, however, is the need for algorithms to infer long-term consequences based on limited, noisy, short-term feedback.
As a result, designing RL algorithms that are both scalable and provably sample efficient has been challenging.

In this work, we revisit the classic linear-program (LP) formulation of RL~\citep{manne1959linear,denardo1968multichain} in an attempt to tackle this long-standing question. 
%
We focus on the associated saddle-point problem of the LP (given by Lagrange duality), which has recently gained traction due to its potential for computationally efficient algorithms with theoretical guarantees~\citep{wang2016online,chen2016stochastic,wang2017randomized,lee2018stochastic,wang2017primal,lin2017revisiting,dai2017boosting,chen2018scalable,lakshminarayanan2018linearly}. 
But in contrast to these previous works based on stochastic approximation,
here we consider a reformulation through the lens of online learning, i.e. regret minimization. Since the pioneering work of~\citet{gordon1999regret,zinkevich2003online}, online learning has evolved into a ubiquitous tool for systematic design and analysis of iterative algorithms.
Therefore, if we can identify a reduction from RL to online learning, we can potentially leverage it to build efficient RL algorithms.

We will show this idea is indeed feasible. We present a reduction by which \emph{any} no-regret online algorithm, after observing $N$ samples, can find a policy $\hat{\pi}_N$ in a policy class $\Pi$ satisfying
$V^{\hat{\pi}_N}(p)  \geq V^{\pi^*}(p) - o(1) - \epsilon_{\Pi}$, where $V^\pi(p)$ is the accumulated reward of policy $\pi$ with respect to some unknown initial state distribution $p$, $\pi^*$ is the optimal policy, and $\epsilon_{\Pi}\geq0$ is a measure of the expressivity of $\Pi$ (see \cref{sec:function approximation} for definition).


Our reduction is built on a refinement of online learning, called Continuous Online Learning (COL), which was proposed to model problems where loss gradients across rounds change continuously with the learner's decisions~\citep{cheng2019online}. COL has a strong connection to equilibrium problems (EPs)~\citep{blum1994optimization,bianchi1996generalized}, and any monotone EP (including our saddle-point problem of interest) can be framed as no-regret learning in a properly constructed COL problem~\cite{cheng2019online}.
Using this idea, our reduction follows naturally by first converting an RL problem to an EP and then the EP to a COL problem. 

Framing RL as COL reveals new insights into the relationship between approximate solutions to the saddle-point problem and approximately optimal policies.
Importantly, this new perspective shows that the RL problem can be separated into two parts: regret minimization and function approximation. 
The first part admits standard treatments from the online learning literature, and the second part can be quantified \emph{independently} of the learning process. 
For example, one can accelerate learning by adopting optimistic online algorithms~\citep{rakhlin2012online,cheng2018predictor} that account for the predictability in COL, without worrying about function approximators.
Because of these problem-agnostic features, the proposed reduction can be used to systematically design efficient RL algorithms with performance guarantees.

As a demonstration, we design an RL algorithm based on arguably the simplest online learning algorithm: mirror descent. 
Assuming a generative model\footnote{In practice, it can be approximated by running a behavior policy with sufficient exploration~\citep{kearns1999finite}.}, 
we prove that, for \emph{any} tabular Markov decision process (MDP), with probability at least $1-\delta$, this algorithm learns an $\epsilon$-optimal policy for the $\gamma$-discounted accumulated reward, using at most
$\tilde{O}\left(\frac{|\SS||\AA|\log(\frac{1}{\delta})}{(1-\gamma)^4\epsilon^2}\right)$ samples, where $|\SS|$,$|\AA|$ are the sizes of state and action spaces, and $\gamma$ is the discount rate.
Furthermore, thanks to the separation property above, our algorithm admits a natural extension with linearly parameterized function approximators, \red{whose sample and per-round computation complexities are \emph{linear} in the number of parameters and independent of $|\SS|$,$|\AA|$, though at the cost of policy performance bias due to approximation error.}


This sample complexity improves the current best provable rate of the saddle-point RL setup~\citep{wang2016online,chen2016stochastic,wang2017randomized,lee2018stochastic} by a large factor of $\frac{|\SS|^2}{(1-\gamma)^2}$, \emph{without} making any assumption on the MDP.\footnote{\citep{wang2017randomized} has the same sample complexity but requires the MDP to be ergodic under any policy.}
\red{This improvement is attributed to our new online-learning-style analysis that uses a cleverly selected comparator in the regret definition.}
While it is possible to devise a minor modification of the previous stochastic mirror descent algorithms, e.g.~\citep{wang2017randomized}, achieving the same rate with our new analysis, we remark that our algorithm is considerably simpler and removes a projection required in previous work~\citep{wang2016online,chen2016stochastic,wang2017randomized,lee2018stochastic}.

Finally, we do note that the same sample complexity can also be achieved, e.g., by 
model-based RL and (phased) Q-learning~\citep{kearns1999finite,kakade2003sample}. However, these methods either have super-linear runtime, with no obvious route for improvement, or could become unstable when using function approximators without further assumption.

\section{SETUP \& PRELIMINARIES}

%
Let $\SS$ and $\AA$ be state and action spaces, which can be discrete or continuous.
We consider $\gamma$-discounted infinite-horizon problems for $\gamma \in [0,1)$. Our goal is  to find a policy $\pi(a|s)$ that maximizes the discounted average return $V^\pi(p) \coloneqq \E_{s\sim p}[V^\pi(s)]$,
where $p$ is the initial state distribution,
\begin{align} \label{eq:value function}
\textstyle 
\hspace{-2mm}
V^{\pi}(s) \coloneqq (1-\gamma)\E_{\xi \sim \rho_\pi(s)}  \left[\sum_{t=0}^{\infty} \gamma^t r(s_t,a_t)  \right] 
\hspace{-1mm}
\end{align}
is the value function of $\pi$ at state $s$, $r:\SS\times\AA\to[0,1]$ is the reward function,
and $\rho_\pi(s)$ is the distribution of trajectory $\xi = s_0, a_0, s_1, \dots$ generated by running $\pi$ from $s_0=s$ in an MDP. 
We assume that the initial distribution $p(s_0)$, the transition $\PP(s'|s,a)$, and the reward function $r(s,a)$ in the MDP are unknown but can be queried through a \emph{generative model}, i.e. we can sample $s_0$ from $p$, $s'$ from $\PP$, and $r(s,a)$ for any $s\in\SS$ and $a\in\AA$. 
%
We remark that the definition of $V^{\pi}$ in \eqref{eq:value function} contains a $(1-\gamma)$ factor. 
We adopt this setup to make writing more compact. We denote the optimal policy as $\pi^*$ and its value function as $V^*$ for short.


\subsection{Duality in RL} \label{sec:duality and stationarity}

Our reduction is based on the linear-program (LP) formulation of RL. We provide a short recap here (please see \cref{app:different RL setups} and \citep{puterman2014markov} for details).

To show how $\max_\pi V^\pi(p)$ can be framed as a LP, let us define 
the
{average} 
 state distribution under $\pi$,
	$
		d^\pi(s) \coloneqq (1-\gamma) \sum_{t=0}^{\infty} \gamma^t d_t^\pi(s)
	$, where $d_t^\pi$ is the state distribution at time $t$ visited by running $\pi$ from $p$ (e.g. $d_0^\pi = p$). By construction, $d^\pi$ satisfies the stationarity property,
\begin{align} \label{eq:stationarity (distribution)}
	d^\pi(s') = (1-\gamma)p(s') + \gamma \E_{s \sim d^\pi} \E_{a\sim \pi |s } [\PP(s'|s,a)].
\end{align}
With $d^\pi$, we can write $V^\pi(p) =\E_{s \sim d^\pi}\E_{a \sim \pi |s}  \left[ r(s,a)  \right] 
$ and our objective  $\max_\pi V^\pi(p)$ equivalently as:
\begin{align} \label{eq:LP (flow)}
\begin{split}
	 &\textstyle \max_{\mub \in \R^{|\SS||\AA|} : \mub \geq \zerob}  \quad  \rb^\t \mub \\
	&\text{s.t.}\quad  
	 \textstyle(1-\gamma)\pb + \gamma \Pb^\t \mub  = \Eb^\t \mub 
\end{split}
\end{align}
where $\rb \in \R^{|\SS||\AA|}$, $\pb \in \R^{|\SS|}$, and $\Pb \in \R^{|\SS||\AA|\times|\SS|}$ are vector forms of $r$, $p$, and $\PP$, respectively, and $\Eb = \Ib \otimes \oneb \in \R^{|\SS||\AA|\times|\SS|}$ (we use $|\cdot|$ to denote the cardinality of a set, $\otimes$ the Kronecker product, $\Ib\in\R^{|\SS|\times|\SS|}$ is the identity, and $\oneb \in \R^{|\AA|}$ the vector of ones). 
In \eqref{eq:LP (flow)}, $\SS$ and $\AA$ 
\blue{may seem to have finite cardinalities}, but the same formulation extends to countable or even continuous spaces (under proper regularity assumptions; see~\citep{hernandez2012discrete}). 
We adopt this abuse of notation (emphasized by bold-faced symbols) for compactness.

The variable $\mub$ of the LP in \eqref{eq:LP (flow)} resembles a joint distribution $d^\pi(s) \pi(a|s)$. To see this, notice that the constraint in~\eqref{eq:LP (flow)} is reminiscent of~\eqref{eq:stationarity (distribution)}, and implies $\norm{\mub}_1 =1 $, i.e. $\mub$ is a probability distribution. Then one can show $\mu(s,a) = d^\pi(s) \pi(a|s)$ when the constraint is satisfied, which implies that \eqref{eq:LP (flow)} is the same as  $\max_\pi V^\pi(p)$ and its solution $\mub^*$ corresponds to $\mu^*(s,a) = d^{\pi^*}(s) \pi^*(a|s)$ of the optimal policy  $\pi^*$.

As \eqref{eq:LP (flow)} is a LP, it suggests looking at its dual, which turns out to be the classic LP formulation of RL\footnote{Our setup in~\eqref{eq:LP (value)} differs from the classic one in the $(1-\gamma)$ factor in the constraint due to the average setup.},
\begin{align} \label{eq:LP (value)}
\begin{split}
	&\textstyle \min_{\vb\in\Rbb^{|\SS|}}\quad  \pb^\t \vb \\
	&\text{s.t.} \quad  (1-\gamma) \rb + \gamma \Pb \vb \leq \Eb \vb.
\end{split}	
\end{align}
It can be verified that for all $\pb >0$, the solution to~\eqref{eq:LP (value)} satisfies the Bellman equation~\citep{bellman1954theory}
and therefore is the optimal value function $\vb^*$ (the vector form of $V^*$).  We note that, for any policy $\pi$, $V^\pi$ by definition satisfies a stationarity property 
\begin{align} \label{eq:statinoarity (value)}
\hspace{-2mm}	
V^\pi(s) =  \E_{a \sim \pi | s} \left[ (1-\gamma)r(s,a) + \gamma \E_{s'\sim \PP |s,a} \left[ V^\pi(s')  \right]  \right] 
\hspace{-1mm}
\end{align}
which can be viewed as a dual equivalent of~\eqref{eq:stationarity (distribution)} for $d^\pi$. Because, for any $s\in\SS$ and $a\in\AA$, $r(s,a)$ is in $[0,1]$, \eqref{eq:statinoarity (value)} implies $V^\pi(s)$ lies in $[0,1]$ too.

\subsection{Toward RL: the Saddle-Point Setup} \label{sec:saddle-point setup}

The LP formulations above require knowing the probabilities $p$ and $\PP$ and are computationally inefficient. 
When only generative models are available (as in our setup), one can alternatively exploit the duality relationship between the two LPs in \eqref{eq:LP (flow)} and~\eqref{eq:LP (value)}, and frame RL as a saddle-point problem~\citep{wang2016online}. 
Let us define
\begin{align} \label{eq:advantage function}
\textstyle
\ab_{\vb} \coloneqq \rb + \frac{1}{1-\gamma}(\gamma \Pb - \Eb )\vb
\end{align}
as the \emph{advantage function} with respect to $\vb$ (where $\vb$ is not necessarily a value function). Then the Lagrangian connecting the two LPs can be written as 
\begin{align} \label{eq:Lagrangian (normalized)}
\LL(\vb, \mub)
\coloneqq \pb^\t \vb +  \mub^\t \ab_\vb,
\end{align}
which leads to the saddle-point formulation,
\begin{align} \label{eq:saddle-point problem of RL}
\min_{\vb \in \VV} \max_{\mub \in \MM} \LL(\vb, \mub),
\end{align}
where the constraints are
\begin{align} 
\VV &= \{\vb \in \R^{|\SS|}:  \vb \geq 0, \norm{\vb}_\infty \leq 1\}  \label{eq:set of value} \\
\MM &= \{\mub\in \R^{|\SS||\AA|}: \mub \geq 0, \norm{\mub}_1 = 1 \}. \label{eq:set of flow}
\end{align}
The solution to \eqref{eq:saddle-point problem of RL} is exactly $(\vb^*, \mub^*)$, but notice that  extra constraints on the norm of $\mub$ and $\vb$ are introduced in $\VV, \MM$, compared with \eqref{eq:LP (flow)} and~\eqref{eq:LP (value)}.
This is a common practice, which uses known bound on the solutions of \eqref{eq:LP (flow)} and~\eqref{eq:LP (value)} (discussed above) to make the search spaces $\VV$ and $\MM$ in \eqref{eq:saddle-point problem of RL} compact and as small as possible so that optimization converges faster.

Having compact variable sets allows using first-order stochastic methods, such as stochastic mirror descent and mirror-prox~\citep{nemirovski2009robust,juditsky2011solving}, to efficiently solve the problem. These methods only require using the generative model to compute unbiased estimates of the gradients $\nabla_\vb \LL = \bb_\mub$ and $\nabla_\mub \LL = \ab_\vb$, where we define
\begin{align} \label{eq:consistency function}
\textstyle \bb_\mub \coloneqq \pb + \frac{1}{1-\gamma}(\gamma\Pb- \Eb)^\t\mub
\end{align}
as the \emph{balance function} with respect to $\mub$. $\bb_\mub$ measures whether $\mub$ violates the stationarity constraint in~\eqref{eq:LP (flow)} and can be viewed as the dual of $\ab_\vb$.
When the state or action space is too large, one can resort to function approximators to represent $\vb$ and $\mub$, which are often realized by linear basis functions for the sake of analysis~\citep{chen2018scalable}.

\subsection{COL and EPs}
\label{sec:COL and EPs}

Finally, we review the COL setup in~\citep{cheng2019online}, which we will use to design the reduction from the saddle-point problem in \eqref{eq:saddle-point problem of RL} to online learning in the next section.

Recall that an online learning problem describes the iterative interactions between a learner and an opponent. In round $n$, the learner chooses a decision $x_n$ from a decision set $\XX$, the opponent chooses a per-round loss function $l_n:\XX \to \R$ based on the learner's decisions, and then information about $l_n$ (e.g. its gradient $\nabla l_n(x_n)$) is revealed to the learner. 
The performance of the learner is usually measured in terms of regret with respect to some $x'\in\XX$,
\begin{align*}
\textstyle
\regret_N(x')\coloneqq \sum_{n=1}^{N} l_n(x_n) -  \sum_{n=1}^N l_n(x').
\end{align*}
When $l_n$ is convex and $\XX$ is compact and convex, many no-regret (i.e. $\regret_N(x') = o(N)$) algorithms are available, such as mirror descent and follow-the-regularized-leader~\citep{cesa2006prediction,shalev2012online,hazan2016introduction}.


COL is a subclass of online learning problems where the loss sequence changes continuously with respect to the played decisions of the learner~\citep{cheng2019online}. 
%
In COL, the opponent is equipped with a bifunction $f:(x,x') \mapsto f_{x}(x')$, where any fixed $x'\in\XX$, 
$\nabla f_x (x')$ is continuous in $x \in \XX$. 
The opponent selects per-round losses based on $f$, but the learner does not know $f$: in round $n$, if the learner chooses $x_n$, the opponent sets
\begin{align} \label{eq:regular per-round loss_main}
l_n(x) = f_{x_n} (x),
\end{align}
and returns, e.g., a stochastic estimate of $\nabla l_n(x_n)$ (the regret is still measured in terms of the noise-free $l_n$).

In \citep{cheng2019online}, a natural connection is shown between COL and equilibrium problems (EPs). As EPs include the saddle-point problem of interest, we can use this idea to turn \eqref{eq:saddle-point problem of RL} into a COL problem.
Recall an EP is defined as follows: Let $\XX$ be compact and $F: (x, x') \mapsto F(x,x')$ be a bifunction s.t. $\forall x,x'\in\XX$, $F(\cdot, x')$ is continuous, $F(x, \cdot)$ is convex, and $F(x,x)\geq0$.\footnote{We restrict ourselves to this convex and continuous case as it is sufficient for our problem setup.} The problem $\EP(\XX,F)$ aims to find $x^\* \in \XX$ s.t.
\begin{align} \label{eq:EP}
F(x^\*,x) \geq 0, \qquad  \forall x \in \XX.
\end{align}
By its definition, a natural residual function to quantify the quality of an approximation solution $x$ to EP is 
$
r_{ep}(x) \coloneqq - \min_{x'\in\XX}F(x,x')  
$
which describes the degree to which~\eqref{eq:EP} is violated at $x$.
We say a bifunction $F$ is \emph{monotone} if, $\forall x,x'\in\XX$, $F(x,x') + F(x',x)  \leq 0$, and \emph{skew-symmetric} if the equality holds.

EPs with monotone bifunctions represent general convex problems, including convex optimization problems, saddle-point problems, variational inequalities, etc. For instance, a convex-concave problem $\min_{y\in\YY} \max_{z\in\ZZ} \phi(y, z)$ can be cast as $\EP(\XX,F)$ with $\XX = \YY\times\ZZ$ and the skew-symmetric bifunction \citep{jofre2014variational}
\begin{align} \label{eq:EP bifunction_main}
F(x,x') \coloneqq  \phi(y', z) - \phi(y, z'),
\end{align}
where $x = (y,z)$ and $x' = (y',z')$. In this case, $r_{ep}(x) = \max_{z'\in\ZZ} \phi(y,z') - \min_{y'\in\YY}\phi(y',z)$ is the duality gap.

\citet{cheng2019online} show that a learner achieves sublinear dynamic regret in COL if and only if the same algorithm can solve $\EP(\XX,F)$ with $F(x,x') = f_x(x') - f_x(x)$. Concretely, they show that, given a monotone $\EP(\XX,F)$ with $F(x,x)=0$ (which is satisfied by \eqref{eq:EP bifunction_main}), one can construct a COL problem by setting $f_{x'}(x) \coloneqq F(x',x)$, i.e. $l_n(x) = F(x_n, x)$, such that any no-regret algorithm can generate an approximate solution to the EP. 
\begin{proposition}\label{pr:regret and residue} 
	{\normalfont\citep{cheng2019online}}
	If $F$ is \emph{skew-symmetric} and $l_n(x) = F(x_n, x)$, then $r_{ep}(\hat{x}_N) \leq \frac{1}{N}\regret_N$, where $\regret_N = \max_{x\in\XX}\regret_N(x)$, and $\hat{x}_N = \frac{1}{N} \sum_{n=1}^N x_n$; the same guarantee holds also for the best decision in $\{x_n\}_{n=1}^N$.
\end{proposition}

%
%
%

\section{AN ONLINE LEARNING VIEW} \label{sec:revist}

We present an alternate online-learning perspective on the saddle-point formulation in \eqref{eq:saddle-point problem of RL}. 
This analysis paves a way for of our reduction in the next section. By reduction, we mean realizing the two steps below: %
\begin{enumerate}
\item
Define a sequence of online losses such that any algorithm with sublinear regret can produce an approximate solution to the saddle-point problem.
\item 
Convert the approximate solution in the first step to an approximately optimal policy in RL.
\end{enumerate}

Methods to achieve these two steps individually are not new.
The reduction from convex-concave problems to no-regret online learning is well known~\citep{abernethy2011blackwell}. Likewise, the relationship between the approximate solution of \eqref{eq:saddle-point problem of RL} and policy performance is also available; this is how the saddle-point formulation~\citep{wang2017randomized} works in the first place. So couldn't we just use these existing approaches? 
We argue that purely combining these two techniques fails to fully capture important structure that resides in RL. 
While this will be made precise in the later analyses, we highlight the main insights here.

Instead of treating \eqref{eq:saddle-point problem of RL} as an \emph{adversarial} two-player online learning problem~\citep{abernethy2011blackwell}, we adopt the recent reduction to COL~\cite{cheng2019online} reviewed in \cref{sec:COL and EPs}.
The main difference is that the COL approach takes a single-player setup and retains the Lipschitz continuity in the source saddle-point problem. 
This single-player perspective is in some sense cleaner and, as we will show in \cref{sec:function approximation}, provides a simple setup to analyze effects of function approximators. 
Additionally, due to continuity, the losses in COL are predictable and therefore make designing fast algorithms possible. 

With the help of the COL reformulation, we study the relationship between the approximate solution to~\eqref{eq:saddle-point problem of RL} 
and the performance of the associated policy in RL. 
We are able to establish a tight bound between the residual and the performance gap, resulting in a large improvement of $\frac{|\SS|^2}{(1-\gamma)^2}$ in sample complexity compared with the best bounds in the literature of the saddle-point setup, \emph{without} adding extra constraints on $\XX$ and  assumptions on the MDP. Overall, this means that \emph{stronger} sample complexity guarantees can be attained by \emph{simpler} algorithms, as we demonstrate in \cref{sec:algorithm demos}.


%
%

The missing proofs of this section are in \cref{app:missing proofs}.

\subsection{The COL Formulation of RL}

First, let us exercise the above COL idea with the saddle-point formulation of RL in \eqref{eq:saddle-point problem of RL}. 
To construct the EP, we can let $\XX = \{x=(\vb, \mub): \vb\in\VV, \mub\in\MM \}$, which is compact.
According to~\eqref{eq:EP bifunction_main}, the bifunction $F$ of the associated $\EP(\XX,F)$ is naturally given as 
\begin{align} \label{eq:bifunction (RL)}
F(x, x') 
&\coloneqq \LL(\vb', \mub) - \LL(\vb, \mub') \nonumber \\
&= \pb^\t \vb' +  \mub^\t \ab_{\vb'} - \pb^\t \vb -  \mub'^\t \ab_\vb
\end{align}
which is skew-symmetric, and $x^* \coloneqq (\vv^*, \mub^*)$ is a solution to $\EP(\XX,F)$.
This identification gives us a COL problem with the loss in the $n$th round defined as
\begin{align} \label{eq:per-round loss of RL}
l_n(x) \coloneqq \pb^\t \vb +  \mub_n^\t \ab_{\vb} - \pb^\t \vb_n -  \mub^\t \ab_{\vb_n}
\end{align}
where $x_n = (\vb_n, \mub_n)$.
We see $l_n$ is a linear loss.
Moreover, because of the continuity in $\LL$, it is {predictable}, i.e. $l_n$ can be (partially) inferred from past feedback as the MDP involved in each round is the same.

\subsection{Policy Performance and Residual}

By \cref{pr:regret and residue}, any no-regret algorithm, when applied to \eqref{eq:per-round loss of RL}, provides guarantees in terms of the residual function $r_{ep}(x)$ of the EP. 
But this is not the end of the story.
We also need to relate the learner decision $x \in \XX$ to a policy $\pi$ in RL and then convert bounds on $r_{ep}(x)$ back to the policy performance $V^\pi(p)$. 
Here we follow the common rule in the literature and associate each $x = (\vb, \mub) \in \XX$ with a policy $\pi_\mub$ defined as
\begin{align}\label{eq:policy from flow}
\pi_\mub(a|s) \propto \mu(s,a).
\end{align}
\blue{In the following, we relate the residual $r_{ep}(x)$ to the performance gap $V^*(p) - V^{\pi_{\mub}}(p)$ 
through a relative performance measure defined as
\begin{align} \label{eq:relative residual}
r_{ep}(x;x') \coloneqq F(x,x) - F(x,x') =  - F(x,x')
\end{align} 
for $x,x'\in\XX$, where the last equality follows from the skew-symmetry of $F$ in \eqref{eq:bifunction (RL)}.
Intuitively, we can view $r_{ep}(x;x')$ as comparing the performance of $x$ with respect to the comparator $x'$ under an optimization problem proposed by $x$, e.g. we have $l_n(x_n) - l_n(x') =  r_{ep}(x_n;x')$. And by the definition in \eqref{eq:relative residual}, it holds that $r_{ep}(x;x') \leq \max_{x'\in\XX}- F(x,x') = r_{ep}(x)$.
}
%

We are looking for inequalities in the form $V^*(p) - V^{\pi_\mub}(p) \leq \kappa(r_{ep}(x;x'))$ that hold for \emph{all} $x \in \XX$ with some strictly increasing function $\kappa$ and some $x'\in\XX$, so we can get \emph{non-asymptotic} performance guarantees once we combine the two steps described at the beginning of this section.
\blue{For example, by directly applying results of \citep{cheng2019online} to the COL in \eqref{eq:per-round loss of RL}, we get
$V^*(p) - V^{\hat{\pi}_N}(p) \leq \kappa(\frac{\regret_N}{N})$, where $\hat{\pi}_N$ is the policy associated with the average/best decision in $\{x_n\}^N_{1 = n}$.}

\subsubsection{The Classic Result} \label{sec:the classic result}

Existing approaches (e.g.~\citep{chen2016stochastic,wang2017randomized,lee2018stochastic}) to the saddle-point point formulation in \eqref{eq:saddle-point problem of RL} rely on the relative residual $r_{ep}(x;x^*)$ with respect to the optimal solution to the problem $x^*$,  which we restate in our notation.
\begin{restatable}{proposition}{RoughResidueBound}
\label{pr:rough residual bound}
For any $x=(\vb,\mub)\in\XX$, if $\Eb^\t \mub \geq (1-\gamma) \pb$,  $r_{ep}(x;x^*) \geq  (1-\gamma) \min_s p(s) \norm{\vb^* - \vb^{\pi_\mub}}_\infty$.
\end{restatable}

Therefore, although the original saddle-point problem in \eqref{eq:saddle-point problem of RL} is framed using $\VV$ and $\MM$, in practice, an extra constraint, such as $\Eb^\t\mub \geq (1-\gamma)\pb$, is added into $\MM$, i.e. these algorithms consider instead  
\begin{align}\label{eq:restricuted set of flow}
\MM' &= \{\mub\in \R^{|\SS||\AA|}: \mub\in\MM, \Eb^\t\mub \geq (1-\gamma)\pb \},
\end{align}
so that the marginal of the estimate $\mub$ can have the sufficient coverage required in \cref{pr:rough residual bound}. This condition is needed to establish non-asymptotic guarantees on the performance of the policy generated by $\mub$~\citep{wang2016online,wang2017randomized,lee2018stochastic}, but it can sometimes be impractical to realize, e.g., when $\pb$ is unknown. Without it, extra assumptions (like ergodicity~\citep{wang2017randomized}) on the MDP are needed.

%
%

However, \cref{pr:rough residual bound} is undesirable for a number of reasons. First, the bound is quite conservative, as it concerns the uniform error $\norm{\vb^* - \vb^{\pi_\mub} }_\infty$ whereas the objective in RL is about the gap $V^*(p) - V^{\pi_{\mub}}(p) =  \pb^\t (\vb^*-\vb^{\pi_\mub})$ with respect to the initial distribution $p$ (i.e. a weighted error).
Second, the constant term  $(1-\gamma) \min_s p(s)$ can be quite small (e.g. when $p$ is uniform, it is $\frac{1-\gamma}{|\SS|}$) which can significantly amplify the error in the residual. 
Because a no-regret algorithm typically decreases the residual in $O(N^{-1/2})$ after seeing $N$ samples, the factor of $\frac{1-\gamma}{|\SS|}$ earlier would turn into a multiplier of $\frac{|\SS|^2}{(1-\gamma)^2}$ in sample complexity.
This makes existing saddle-point approaches sample inefficient in comparison with other RL methods like Q-learning~\citep{kakade2003sample}. 
Lastly, enforcing $\Eb^\t \mub \geq (1-\gamma) \pb$ requires knowing $\pb$ (which is unavailable in our setup) and adds extra projection steps during optimization. When $\pb$ is unknown, while it is possible to modify this constraint to use a uniform distribution, this might worsen the constant factor and could introduce bias.

One may conjecture that the bound in \cref{pr:rough residual bound} could perhaps be tightened by better analyses. However, we prove this is impossible in general. 
\begin{restatable}{proposition}{ResidueLowerBound} \label{pr:residual lower bound}
There is a class of MDPs such that, for some $x\in\XX$, \cref{pr:rough residual bound} is an equality.
\end{restatable}
We note that \cref{pr:residual lower bound} does not hold for \emph{all} MDPs.
Indeed, if one makes stronger assumptions on the MDP, such as that the Markov chain induced by \emph{every} policy is ergodic~\citep{wang2017randomized}, then it is possible to show, for all $x\in\XX$, $r_{ep}(x;x^*) = c \norm{\vb^* - \vb^{\pi_\mub}}_\infty$ for some constant $c$ independent of $\gamma$ and $|\SS|$, when one constrains $\Eb^\t\mub \geq (1-\gamma + \gamma \sqrt{c}) \pb$. Nonetheless, this construct still requires adding an undesirable constraint to $\XX$.

\subsubsection{Curse of Covariate Shift}

Why does this happen? We can view this issue as a form of \emph{covariate shift}, i.e. a mismatch between distributions. To better understand it, we notice a simple equality, which has often been used implicitly, e.g. in the technical proofs of \cite{wang2017randomized}. 
\begin{restatable}{lemma}{ResidueEquality}\label{lm:residual}
	For any $x = (\vb, \mub)$, if $x'\in\XX$ satisfies~\eqref{eq:stationarity (distribution)} and~\eqref{eq:statinoarity (value)} (i.e. $\vb'$ and $\mub'$ are the value function and state-action distribution of policy $\pi_{\mub'}$), $r_{ep}(x;x') = - \mub^\t \ab_{\vb'} $.
\end{restatable}
\cref{lm:residual} implies $r_{ep}(x;x^*) =  - \mub^\t \ab_{\vb^*}$, which is non-negative. This term is similar to an equality called the performance difference lemma~\cite{ng1999policy,kakade2002approximately}.
\begin{restatable}{lemma}{PerformanceDifferenceLemma} \label{lm:performance difference lemma}
	Let $\vb^\pi$ and $\mub^\pi$ denote the value and state-action distribution of some policy $\pi$. Then for \emph{any} function $\vb'$, it holds that
	$
	\pb^\t (\vb^\pi - \vb')
	= (\mub^\pi)^\t \ab_{\vb'}
	$. In particular, it implies $V^\pi(p) - V^{\pi'}(p) =  (\mub^\pi)^\t \ab_{\vb^{\pi'}}$.
\end{restatable}
From \cref{lm:residual,lm:performance difference lemma}, we see that the difference between the residual $r_{ep}(x;x^*) =  - \mub^\t \ab_{\vb^*}$ and the performance gap  $V^{\pi_\mub}(p) - V^{\pi^*}(p) = (\mub^{\pi_\mub})^\t \ab_{\vb^*}$ is due to the mismatch between $\mub$ and $\mub^{\pi_\mub}$, or more specifically, the mismatch between the two marginals $\db = \Eb^\t \mub$ and $\db^{\pi_\mub} = \Eb^\t\mub^{\pi_\mub}$. 
Indeed, when $\db = \db^{\pi_\mub}$, the residual is equal to the performance gap. However, in general, we do not have control over that difference for the sequence of variables $\{x_n = (\vb_n,\mub_n) \in \XX \}$ an algorithm generates. 
The sufficient condition in \cref{pr:rough residual bound} attempts to mitigate the difference, using the fact $\db^{\pi_\mub} = (1-\gamma) \pb + \gamma \Pb_{\pi_\mub}^\t \db^{\pi_\mub}$ from \eqref{eq:stationarity (distribution)}, where $\Pb_{\pi_\mub}$ is the transition matrix under $\pi_\mub$. But the missing half $\gamma \Pb_{\pi_\mub}^\t \db^{\pi_\mub}$ (due to the long-term effects in the MDP) introduces the unavoidable, weak constant $ (1-\gamma) \min_s p(s)$, if we want to have an uniform bound on $\norm{\vb^* - \vb^{\pi_\mub}}_\infty$.
The counterexample in \cref{pr:residual lower bound} was designed to maximize the effect of covariate shift, so that $\mub$ fails to captures state-action pairs with high advantage. To break the curse, we must properly weight the gap between $\vb^*$ and $\vb^{\pi_\mub}$ instead of relying on the uniform bound on $\norm{\vb^* - \vb^{\pi_\mub}}_\infty$ as before.


\section{THE REDUCTION} \label{sec:the reduction}

The analyses above reveal both good and bad properties of the saddle-point setup in \eqref{eq:saddle-point problem of RL}. 
On the one hand, we showed that approximate solutions to the saddle-point problem in \eqref{eq:saddle-point problem of RL} can be obtained by running any no-regret algorithm in the single-player COL problem defined in \eqref{eq:per-round loss of RL}; many efficient algorithms are available from the online learning literature. On the other hand, we also discovered a root difficulty in converting an approximate solution of \eqref{eq:saddle-point problem of RL} to an approximately optimal policy in RL (\cref{pr:rough residual bound}), even after imposing strong conditions like \eqref{eq:restricuted set of flow}. 
At this point, one may wonder if the formulation based on \eqref{eq:saddle-point problem of RL} is fundamentally sample inefficient compared with other approaches to RL, but this is actually not true.

Our main contribution shows that learning a policy through running a no-regret algorithm in the COL problem in \eqref{eq:per-round loss of RL} is, in fact, as sample efficient in policy performance as other RL techniques, even without the common constraint in \eqref{eq:restricuted set of flow} or extra assumptions on the MDP like ergodicity imposed in the literature.
\begin{theorem}\label{th:reduction of RL}
	Let $X_N = \{x_n\in\XX\}_{n=1}^N$ be \emph{any} sequence. Let $\hat{\pi}_N$ be the policy given by $\hat{x}_N$ via \eqref{eq:policy from flow}, which is either the average or the best decision in $X_N$. Define $y_N^* \coloneqq (\vb^{\hat{\pi}_N}, \mub^*)$. Then
$
\textstyle
V^{\hat{\pi}_N}(p)  \geq V^*(p)  -  \frac{\regret_N(y_N^*)}{N}
$.
\end{theorem}
\cref{th:reduction of RL} shows that if $X_N$ has sublinear regret, then both the average policy and the best policy in $X_N$ converge to the optimal policy in performance with a rate $O(\regret_N(y_N^*)/N)$. Compared with existing results obtained through \cref{pr:rough residual bound}, the above result removes the factor $(1-\gamma) \min_s p(s)$ and impose \emph{no} assumption on $X_N$ or the MDP.
Indeed \cref{th:reduction of RL} holds for \emph{any} sequence. For example, when $X_N$ is generated by stochastic feedback of $l_n$, \cref{th:reduction of RL} continues to hold, as the regret is defined in terms of $l_n$, not of the sampled loss.
Stochasticity only affects the regret rate. 

In other words, we have shown that when $\mub$ and $\vb$ can be directly parameterized, an approximately optimal policy for the RL problem can be obtained by running any no-regret online learning algorithm, and that the policy quality is simply dictated by the regret rate. 
To illustrate, in \cref{sec:algorithm demos} we will prove that simply running mirror descent in this COL produces an RL algorithm that is as  sample efficient as other common RL techniques.
One can further foresee that algorithms leveraging the continuity in COL---e.g. mirror-prox~\citep{juditsky2011solving} or PicCoLO~\citep{cheng2018predictor}---and variance reduction can lead to more sample efficient RL algorithms.

Below we will also demonstrate how to use the fact that COL is \emph{single-player} (see \cref{sec:COL and EPs}) to cleanly incorporate the effects of using function approximators to model $\mub$ and $\vb$. We will present a corollary of \cref{th:reduction of RL}, which separates the problem of \emph{learning} $\mub$ and $\vb$, and that of \emph{approximating} $\MM$ and $\VV$ with function approximators. 
The first part is controlled by the rate of regret in online learning, and the second part depends on only the chosen class of function approximators, independently of the learning process.
As these properties are agnostic to problem setups and algorithms, our reduction leads to a framework for systematic synthesis of new RL algorithms with performance guarantees.
The missing proofs of this section are in \cref{app:missing proofs of reduction}.

\subsection{Proof of \cref{th:reduction of RL}} \label{sec:proof of reduction}

The main insight of our reduction is to adopt, in defining $r_{ep}(x;x')$, a comparator $x'\in\XX$ based on the output of the algorithm (represented by $x$), instead of the fixed comparator $x^*$ (the optimal pair of value function and state-action distribution) that has been used conventionally, e.g. in \cref{pr:rough residual bound}. 
While this idea seems unnatural from the standard saddle-point or EP perspective, it is possible, because the regret in online learning is measured against the worst-case choice in $\XX$, which is allowed to be selected in \emph{hindsight}. 
Specifically, we propose to select the following comparator to directly bound $ V^*(p) - V^{\hat{\pi}_N}(p)$ instead of the conservative measure $\norm{V^* - V^{\hat{\pi}_N}}_\infty$ used before.
\begin{restatable}{proposition}{CleverResidueBound}
\label{pr:clever residual bound}
For $x = (\vb, \mub)\in\XX$, define $y_x^* \coloneqq (\vb^{\pi_\mub}, \mub^*) \in \XX$. It holds $r_{ep}(x;y_x^*) = V^*(p) - V^{\pi_\mub}(p)$.
\end{restatable}
To finish the proof, let $\hat{x}_N $ be either $\frac{1}{N} \sum_{n=1}^N x_n$ or \red{$\argmin_{x\in X_N} r_{ep}(x;y_x^*)$,} and let $\hat{\pi}_N$ denote the policy given by~\eqref{eq:policy from flow}.
First, $V^*(p) - V^{\hat{\pi}_N}(p) = r_{ep}(\hat{x}_N; y_N^*)$ by \cref{pr:clever residual bound}. \red{Next we follow the proof idea of \cref{pr:regret and residue} in~\citep{cheng2019online}: because $F$ is skew-symmetric and $F(y_N^*, \cdot)$ is convex,  we have by \eqref{eq:relative residual}}
\begin{align*}
&V^*(p) - V^{\hat{\pi}_N}(p) = r_{ep}(\hat{x}_N; y_N^*) = -F(\hat{x}_N,y_N^*)\\
&\textstyle  = F(y_N^*, \hat{x}_N)\leq \frac{1}{N} \sum_{n=1}^N F(y_N^*, x_n)  \\
&= \textstyle \frac{1}{N} \sum_{n=1}^N - F(x_n, y_N^*)  = \frac{1}{N} \regret_N(y_N^*).
\end{align*}

%

\subsection{Function Approximators}  \label{sec:function approximation}

When the state and action spaces are large or continuous, directly optimizing $\vb$ and $\mub$ can be impractical. Instead we can consider optimizing over a subset of feasible choices parameterized by function approximators
\begin{align} \label{eq:function approximators}
\XX_\Theta = \{ \xb_\theta = (\phib_{\theta}, \psib_\theta):  \psib_\theta \in \MM,  \theta \in \Theta \},
\end{align}
where $\phib_\theta$ and $\psib_\theta$ are functions parameterized by $\theta \in \Theta$, and $\Theta$ is a parameter set. 
Because COL is a single-player setup, we can extend the previous idea and \cref{th:reduction of RL} to provide performance bounds in this case by a simple rearrangement (see \cref{app:missing proofs of reduction}), which is a common trick used in the online imitation learning literature~\citep{ross2011reduction,cheng2018convergence,cheng2018accelerating}.
Notice that, in \eqref{eq:function approximators}, we require only $\psib_\theta \in \MM$, but not $\phib_{\theta} \in \VV$, because for the performance bound in our reduction to hold, we only need the constraint $\MM$ (see \cref{lm:advantage bound} in proof of \cref{pr:clever residual bound}).%
\begin{restatable}{corollary}{ReductionForFunctionApporximator}\label{cr:reduction for funcapp}
Let $X_N = \{x_n\in\XX_\theta \}_{n=1}^N$ be any sequence. Let $\hat{\pi}_N$ be the policy given either by the average or the best decision in  $X_N$. It holds that 
\begin{align*}
\textstyle
V^{\hat{\pi}_N}(p)  \geq V^*(p) - \frac{\normalfont{\regret}_N(\Theta)}{N} - \epsilon_{\Theta,N}
\end{align*}
where
$\epsilon_{\Theta,N} = \min_{x_\theta\in\XX_{\theta}}  r_{ep}(\hat{x}_N;y_N^*) - r_{ep}(\hat{x}_N; x_\theta)$ measures the expressiveness of $X_\theta$, and 
 $\normalfont{\regret}_N(\Theta) \coloneqq \sum_{n=1}^{N} l_n(x_n) - \min_{x\in\XX_\Theta} \sum_{n=1}^N l_n(x)$.
\end{restatable}

We can quantify $\epsilon_{\Theta,N}$ with the basic H\"older's inequality.%
\begin{restatable}{proposition}{SizeOfEpsilon} \label{pr:size of epsilon}
Let $\hat{x}_N = (\hat{\vb}_N, \hat{\mub}_N)$.
Under the setup in \cref{cr:reduction for funcapp}, \emph{regardless} of the parameterization, it is true that $\epsilon_{\Theta,N}$ is no larger than
\begin{small}
\begin{align*}
&\hspace{-1mm}
\min_{(\vb_\theta,\mub_\theta) \in \XX_\Theta} 
\hspace{-1mm}
\frac{\norm{\mub_\theta - \mub^*}_1}{1-\gamma}  + 
\min_{\wb:\wb\geq 1} \norm{\bb_{\hat{\mub}_N}}_{1,\wb} \norm{\vb_\theta - \vb^{\hat{\pi}_N}}_{\infty,1/\wb}
\\
&\leq 
\min_{(\vb_\theta,\mub_\theta) \in \XX_\Theta } 
\frac{1}{1-\gamma} \left(  \norm{\mub_\theta - \mub^*}_1
+ 2 \norm{\vb_\theta - \vb^{\hat{\pi}_N}}_\infty \right).
\end{align*}
\end{small}%
where the norms are defined as $\norm{\xb}_{1,\wb} = \sum_i w_i |x_i|$ and $\norm{\xb}_{\infty,1/\wb} = \max_i w_i^{-1} |x_i|$.	
\end{restatable}
\cref{pr:size of epsilon} says $\epsilon_{\Theta,N}$ depends on how well $\XX_\Theta$ captures the value function of the output policy $\vb^{\hat{\pi}_N}$  and the optimal state-action distribution $\mub^*$. 
We remark that this result is independent of how $\vb^{\hat{\pi}_N}$ is generated. 
Furthermore, \cref{pr:size of epsilon} makes \emph{no} assumption whatsoever on the structure of function approximators. It even allows sharing parameters $\theta$ between $\vb =\phib_{\theta}$ and $ \mub = \psib_\theta$, e.g., they can be a bi-headed neural network, which is common for learning shared feature representations. More precisely, the structure of the function approximator would only affect whether $l_n((\phib_\theta,\psib_\theta))$ remains a convex function in $\theta$, which determines the difficulty of designing algorithms with sublinear regret. 

In other words, the proposed COL formulation provides a reduction which dictates the policy performance with two separate factors: 1) the rate of regret $\regret_N(\Theta)$ which is controlled by the choice of online learning algorithm; 2) the approximation error $\epsilon_{\Theta,N}$ which is determined by the choice of function approximators. These two factors can almost be treated independently, except that the choice of function approximators would determine the properties of $l_n((\phib_\theta,\psib_\theta))$ as a function of $\theta$, and the choice of $\Theta$ needs to ensure \eqref{eq:function approximators} is admissible.

\section{SAMPLE COMPLEXITY OF MIRROR DESCENT}
\label{sec:algorithm demos}

\begin{algorithm}[t] 
	{\small
		\caption{Mirror descent for RL}\label{alg:md for RL} 
		\begin{algorithmic} [1]
			\renewcommand{\algorithmicensure}{\textbf{Input:}}		
			\renewcommand{\algorithmicrequire}{\textbf{Output:}}
			\ENSURE $\epsilon$ optimality of the $\gamma$-average return\\
			 \hspace{5mm} $\delta$ maximal failure probability\\
			\hspace{5mm} generative model of an MDP 
			\REQUIRE $\hat{\pi}_N = \pi^{\hat{\mub}_N}$ 
			\STATE  $x_1 = (\vb_1, \mub_1)$ where $\mub_1$ is uniform and  $\vb_1 \in\VV$
			\STATE Set $N = \tilde{\Omega}(\frac{|\SS||\AA|\log(\frac{1}{\delta})}{(1-\gamma)^2\epsilon^2})$ and $\eta = (1-\gamma)(|\SS||\AA|N)^{-1/2}$
			\STATE Set the Bregman divergence as \eqref{eq:Bregman divergence choice}
			\FOR {$n = 1\dots N-1$}
			\STATE Sample $g_n$ according to \eqref{eq:stochastic gradient estimate}
			\STATE Update to $x_{n+1}$ according to \eqref{eq:mirror descent}
			\\			
			\ENDFOR
			\STATE Set $ (\hat{\vb}_N, \hat{\mub}_N) = \hat{x}_N  = \frac{1}{N}\sum_{n=1}^{N} x_n$
		\end{algorithmic}
	}
\end{algorithm}

We demonstrate the power of our reduction by applying perhaps the simplest online learning algorithm, mirror descent, to the proposed COL problem in \eqref{eq:per-round loss of RL} with stochastic feedback (\cref{alg:md for RL}).
For transparency, we discuss the tabular setup. We will show a natural extension to basis functions at the end.

Recall that mirror descent is a first-order algorithm, whose update rule can be written as
\begin{align} \label{eq:mirror descent}
\textstyle
x_{n+1} = \argmin_{x\in\XX} \lr{g_n}{x} + \frac{1}{\eta} B_{R}(x||x_n)
\end{align}
where $\eta >0$ is the step size, $g_n$ is the feedback direction, and $B_R(x||x') = R(x) - R(x') - \lr{\nabla R(x')}{x - x'}$ is the Bregman divergence generated by a strictly convex function $R$.
Based on the geometry of $\XX = \VV \times \MM$, we consider a natural Bregman divergence of the form
\begin{align} \label{eq:Bregman divergence choice}
\textstyle
B_R(x'||x) =  \frac{1}{2|\SS|} \norm{\vb' - \vb}_2^2 + KL(\mub'||\mub) 
\end{align}
This choice mitigates the effects of dimension (e.g. if we set $x_1 = (\vb_1, \mub_1)$ with $\mub_1$ being the uniform distribution, it holds
$
B_R(x'||x_1) = \tilde{O}(1)
$ for any $x' \in\XX$).

To define the feedback direction $g_n$, we slightly modify the per-round loss $l_n$ in \eqref{eq:per-round loss of RL} and consider a new loss 
\begin{align} \label{eq:modified per-round loss}
\textstyle
h_n(x) \coloneqq \bb_{\mub_n}^\t \vb  + \mub^\t (\frac{1}{1-\gamma} \oneb - \ab_{\vb_n})
\end{align}
that shifts $l_n$ by a constant, where $\oneb$ is the vector of ones. 
One can verify that $l_n(x)-l_n(x') = h_n(x) - h_n(x')$,  for all $x,x' \in \XX$ \blue{when $\mub, \mub'$ in $x$ and $x'$ satisfy $\norm{\mub}_1 = \norm{\mub'}_1$ (which holds for \cref{alg:md for RL}).} Therefore, using $h_n$ does not change regret. The reason for using $h_n$ instead of $l_n$ is to make $\nabla_\mub h_n((\vb,\mub))$ (and its unbiased approximation) a positive vector, 
so the regret bound can have a better dimension dependency. This is a common trick used in online learning (e.g. EXP3 \cite{journals/ftopt/Hazan16}) for optimizing variables living in a simplex ($\mub$ here).

We set the first-order feedback $g_n$ as an unbiased \emph{sampled} estimate of $\nabla h_n(x_n)$. In round $n$, this is realized by two independent calls of the generative model:
\begin{align} \label{eq:stochastic gradient estimate}
\hspace{-2mm} g_n 
%
= \begin{bmatrix}
\tilde{\pb}_n + \frac{1}{1-\gamma}(\gamma\tilde{\Pb}_n -\Eb_n)^\t \tilde{\mub}_n \\
|\SS||\AA| (\frac{1}{1-\gamma
} \hat{\oneb}_n - \hat{\rb}_n - \frac{1}{1-\gamma}(\gamma\hat{\Pb}_n-\hat{\Eb}_n)\vb_n)
\end{bmatrix}
\end{align}
Let $g_n = [\gb_{n,v}; \gb_{n,\mu}]$.
For $\gb_{n,v}$, we sample $\pb$,  sample $\mub_n$ to get a state-action pair, and query the transition $\Pb$ at the state-action pair sampled from $\mub_n$. ($\tilde{\pb}_n$, $\tilde{\Pb}_n$, and $\tilde{\mub}_n$ denote the single-sample estimate of these probabilities.)
For $\gb_{n,\mu}$, we first sample \emph{uniformly} a state-action pair (which explains the factor $|\SS||\AA|$), and then query the reward $\rb$ and the transition $\Pb$. ($\hat{\oneb}_n$, $\hat{\rb}_n$, $\hat{\Pb}_n$, and $\hat{\Eb}_n$ denote the single-sample estimates.) To emphasize, we use $\tilde{\cdot}$ and $\hat{\cdot}$ to distinguish the empirical quantities obtained by these  two independent queries. By construction, we have $\gb_{n,\mu} \geq 0$.
It is clear that this direction $g_n$ is unbiased, i.e. $\E[g_n] = \nabla h_n(x_n)$. Moreover, it is extremely sparse and can be computed using $O(1)$ sample, computational, and memory complexities.

Below we show this algorithm, despite being extremely simple, has strong theoretical guarantees. In other words, we obtain simpler versions of the algorithms proposed in \citep{wang2016online,wang2017randomized,chen2018scalable} but with improved performance. 
\begin{restatable}{theorem}{SampleComplexityOfMirrorDescent} \label{th:sample complexity of mirror descent}
With probability  $1-\delta$, \cref{alg:md for RL} learns an $\epsilon$-optimal policy with  
$
\tilde{O}\left(\frac{|\SS||\AA|\log(\frac{1}{\delta})}{(1-\gamma)^2\epsilon^2}\right)
$ samples.
\end{restatable}%
Note that the above statement makes no assumption on the MDP (except the tabular setup for simplifying analysis). 
Also, because the definition of value function in \eqref{eq:value function} is scaled by a factor $(1-\gamma)$, the above result translates into a sample complexity  in  $\tilde{O}\left(\frac{|\SS||\AA|\log(\frac{1}{\delta})}{(1-\gamma)^4\epsilon^2}\right)$ for the conventional discounted accumulated rewards.

\subsection{Proof Sketch of \cref{th:sample complexity of mirror descent}} \label{sec:proof sketch of sample complexity of mirror descent}

The proof is based on the basic property of mirror descent and martingale concentration. We provide a sketch here; please refer to \cref{app:proof of sample complexity of mirror descent} for details. 
Let $y_N^* = (\vb^{\hat{\pi}_N}, \mub^*)$.
We bound the regret in \cref{th:reduction of RL}  by the following rearrangement, where the first equality below is because $h_n$ is a constant shift from $l_n$. 
\begin{small}
\begin{align*} 
&\regret_N (y_N^*)
=  \sum_{n=1}^{N} h_n(x_n) - \sum_{n=1}^N h_n(y_N^* ) \nonumber \\
&\leq  
\left(\sum_{n=1}^{N} (\nabla h_n(x_n)-g_n)^\t x_n \right)  
 + \left( \max_{x\in\XX} \sum_{n=1}^N g_n^\t (x_n -x) \right)
\\ &\quad + 
\left( \sum_{n=1}^N   (g_n - \nabla h_n(x_n))^\t y_N^* \right)
\end{align*}
\end{small}%
We recognize the first term is a martingale, because $x_n$ does not depend on $g_n$. Therefore, we can appeal to a Bernstein-type martingale concentration and prove it is in $\tilde{O}( \frac{\sqrt{N|\SS||\AA|\log(\frac{1}{\delta})}}{1-\gamma} )$.  For the second term, by treating $g_n^\t x$ as the per-round loss, we can use standard regret analysis of mirror descent and show a bound in $\tilde{O}(\frac{\sqrt{N|\SS||\AA|}}{1-\gamma})$. 
For the third term, because $\vb^{\hat{\pi}_N}$ in $y_N^* = (\vb^{\hat{\pi}_N}, \mub^*)$ depends on $\{g_n\}_{n=1}^N$, it is \emph{not} a martingale. Nonetheless, we are able to handle it through a union bound and show it is again no more than $\tilde{O}( \frac{\sqrt{N|\SS||\AA|\log(\frac{1}{\delta})}}{1-\gamma})$. Despite the union bound, it does not increase the rate because we only need to handle $\vb^{\hat{\pi}_N}$, not $\mub^*$ which induces a martingale.
To finish the proof, we substitute this high-probability regret bound into \cref{th:reduction of RL} to obtain the desired claim.

\subsection{Extension to Function Approximators}

The above algorithm assumes the tabular setup for illustration purposes.
\red{In \cref{app:sample complexity using function approximators}, we describe a direct extension of \cref{alg:md for RL} that uses linearly parameterized function approximators of the form $x_\theta = (\Phib\thetab_v, \Psib\thetab_\mu)$, where columns of bases $\Phib,\Psib$ belong to $\VV$ and $\MM$, respectively, and $(\thetab_v,\thetab_\mu) \in \Theta$. 

Overall the algorithm stays the same, except the gradient is computed by chain-rule, which can be done in $O(dim(\Theta))$ time and space. While this seems worse, the computational complexity per update actually improves to $O(dim(\Theta))$ from the slow $O(|\SS||\AA|)$ (required before for the projection in \eqref{eq:mirror descent}), as now we only optimize in $\Theta$. 
Moreover, we prove that its sample complexity is also better, though at the cost of bias $\epsilon_{\Theta,N}$ in \cref{cr:reduction for funcapp}. Therefore, the algorithm becomes applicable to large-scale or continuous problems.
\begin{restatable}{theorem}{SampleComplexityOfMirrorDescentWithBasis} \label{th:sample complexity of mirror descent with basis}
	Under a proper choice of $\Theta$ and $B_R$, with probability  $1-\delta$, \cref{alg:md for RL} learns an $(\epsilon+\epsilon_{\Theta,N})$-optimal policy with  
	$
	\tilde{O}\left(\frac{dim(\Theta)\log(\frac{1}{\delta})}{(1-\gamma)^2\epsilon^2}\right)
	$ samples.
\end{restatable}%
The proof is in \cref{app:sample complexity using function approximators}, and mainly follows \cref{sec:proof sketch of sample complexity of mirror descent}. First, we choose some $\Theta$ to satisfy \eqref{eq:function approximators} so we can use \cref{cr:reduction for funcapp} to reduce the problem into regret minimization.
To make the sample complexity independent of $|\SS|$,$|\AA|$, the key is to uniformly sample over the columns of $\Psib$ (instead of over all states and actions like \eqref{eq:stochastic gradient estimate}) when computing unbiased estimates of $\nabla_{\thetab_\mu} h_n((\thetab_v,\thetab_\mu))$. The intuition is that we should only focus on the places our basis functions care about (of size $dim(\Theta)$), instead of wasting efforts to visit all possible combinations (of size $|\SS||\AA|$). 
}

\section{CONCLUSION}

We propose a reduction from RL to no-regret online learning 
that 
provides a systematic way to design new RL algorithms with performance guarantees. 
Compared with existing approaches, our framework makes no assumption on the MDP and naturally works with function approximators.
To illustrate, we design a simple RL algorithm based on mirror descent; it achieves similar sample complexity as other RL techniques, but uses minimal assumptions on the MDP and is scalable to large or continuous problems.
This encouraging result evidences the strength of the online learning perspective. 
As a future work, we believe even faster learning in RL is possible by leveraging control variate for variance reduction and by applying more advanced online techniques~\citep{rakhlin2012online,cheng2018predictor} that exploit the continuity in COL to predict the future gradients.
%

\subsubsection*{Acknowledgements}

This research is partially supported by NVIDIA Graduate Fellowship.

\bibliography{ref}



\clearpage
\onecolumn
\appendix
{\centering{\large\textbf{Appendix}}\par}

\section{Review of RL Setups} \label{app:different RL setups}

We provide an extended review of different formulations of RL for interested readers.
First, let us recall the problem setup. 
Let $\SS$ and $\AA$ be state and action spaces, and let $\pi(a|s)$ denote a policy. 
For $\gamma \in [0,1)$, we are interested in solving a $\gamma$-discounted infinite-horizon RL problem:
\begin{align} \label{eq:canonical RL}
\textstyle
\max_\pi V^\pi(p) , \qquad\text{s.t.}\quad V^\pi(p) \coloneqq (1-\gamma) \E_{s_0 \sim p}\E_{\xi \sim \rho_\pi(s_0)}  \left[\sum_{t=0}^{\infty} \gamma^t r(s_t,a_t)  \right]
\end{align}
where $V^\pi(p)$ is the discounted average return, $r:\SS\times\AA\to[0,1]$ is the reward function,
$\rho_\pi(s_0)$ denotes the distribution of trajectory $\xi = s_0, a_0, s_1, \dots$ generated by running $\pi$ from state $s_0$ in a Markov decision process (MDP), 
and $p$ is a fixed but unknown initial state distribution.

\subsection{Coordinate-wise Formulations} \label{sec:coordinate-wise formulation}

\paragraph{RL in terms of stationary state distribution}
Let $d_t^\pi(s)$ denote the state distribution at time $t$ given by running $\pi$ starting from $p$. 
We define its $\gamma$-weighted mixture as
\begin{align} \label{generalized stationary state distribution}
\textstyle
d^\pi(s) \coloneqq (1-\gamma) \sum_{t=0}^{\infty} \gamma^t d_t^\pi(s)
\end{align}
We can view $d^\pi$ in~\eqref{generalized stationary state distribution} as a {form of stationary state distribution of $\pi$}, because it is a valid probability distribution of state and satisfies the stationarity property below,
\begin{align} \tag{\ref{eq:stationarity (distribution)}}
d^\pi(s') = (1-\gamma)p(s') + \gamma \E_{s \sim d^\pi} \E_{a\sim \pi |s } [\PP(s'|s,a)]
\end{align}
where $\PP(s'|s,a)$ is the transition probability of the MDP. 
The definition in~\eqref{generalized stationary state distribution} generalizes the concept of stationary distribution of MDP; 
as $\gamma\to1$, $d^\pi$ is known as the limiting average state distribution, which is the same as the stationary distribution of the MDP under $\pi$, if one exists.
Moreover, with the property in~\eqref{eq:stationarity (distribution)}, $d^\pi$ summarizes the Markov structure of RL, and allows us to 
write \eqref{eq:canonical RL} simply as
\begin{align} \label{eq:RL in stationary distribution}
\max_{\pi} V^\pi(p), \qquad \text{s.t.}\quad V^\pi(p) =\E_{s \sim d^\pi}\E_{a \sim \pi |s}  \left[ r(s,a)  \right] 
\end{align}
after commuting the order of expectation and summation. That is, an RL problem aims to maximize the expected reward under the stationary state-action distribution generated by the policy $\pi$.

\paragraph{RL in terms of value function}
We can also write~\eqref{eq:canonical RL} in terms of value function. 
Recall 
\begin{align} \tag{\ref{eq:value function}}
\textstyle
V^\pi(s) \coloneqq  (1-\gamma) \E_{\xi \sim \rho_\pi(s_0)|s_0=s }\left[\sum_{t=0}^{\infty} \gamma^t r(s_t,a_t)  \right] 
\end{align}
is the value function of $\pi$. 
By definition, $V^\pi$ (like $d^\pi$) satisfies a stationarity property 
\begin{align} \tag{\ref{eq:statinoarity (value)}}
V^\pi(s) =  \E_{a \sim \pi | s} \left[ (1-\gamma)r(s,a) + \gamma \E_{s'\sim\PP|s,a} \left[ V^\pi(s')  \right]  \right]
\end{align}
which can be viewed as a dual equivalent of~\eqref{eq:stationarity (distribution)}. Because $r$ is in $[0,1]$, \eqref{eq:statinoarity (value)} implies $V^\pi$ lies in $[0,1]$. 

The value function $V^*$ (a shorthand of $V_{\pi^*}$) of the optimal policy $\pi^*$ of the RL problem satisfies the so-called Bellman equation~\citep{bellman1954theory}:
$
V^*(s) =  \max_{a\in\AA}  (1-\gamma)r(s,a) + \gamma \E_{s'\sim\PP|s,a} \left[ V^*(s')  \right]  
$, 
where the optimal policy $\pi^*$ can be recovered as the $\argmax$. 
Equivalently, by the definition of $\max$, the Bellman equation amounts to finding the smallest $V$ such that 
$
V(s) \geq  (1-\gamma)r(s,a) + \gamma \E_{s'\sim\PP|s,a} \left[ V(s')  \right]
$, $\forall s\in \SS, a \in \AA$.
In other words, the RL problem in~\eqref{eq:canonical RL} can be written as 
\begin{align} \label{eq:value-based RL}
\min_{V} \E_{s \sim p} [V(s)] \qquad \text{s.t.}\quad V(s) \geq  (1-\gamma)r(s,a) + \gamma \E_{s'\sim\PP|s,a} \left[ V(s')  \right],   \qquad \forall s\in \SS, a \in \AA
\end{align}

\subsection{Linear Programming Formulations} \label{sec:LP for RL}

We now connect the above two alternate expressions through the classical LP setup of RL~\cite{manne1959linear,denardo1968multichain}. 

\paragraph{LP in terms of value function}
The classic LP formulation\footnote{Our setup in~\eqref{eq:LP (value)} differs from the classic one in the $(1-\gamma)$ factor in the constraint to normalize the problem.} is simply a restatement of~\eqref{eq:value-based RL}: 
\begin{align} \tag{\ref{eq:LP (value)}}
\min_{\vb}\quad  \pb^\t \vb  \qquad 
\text{s.t.} \quad  (1-\gamma) \rb + \gamma \Pb \vb \leq \Eb \vb
\end{align}
where $\pb \in \R^{|\SS|}$, $\vb \in \R^{|\SS|}$, and $\rb \in \R^{|\SS||\AA|}$ are the vector forms of $p$, $V$, $r$, respectively,  $\Pb \in \R^{|\SS||\AA|\times|\SS|}$ is the transition probability\footnote{We arrange the coordinates in a way such that along the $|\SS||\AA|$ indices are contiguous in actions.}, and $\Eb = \Ib \otimes \oneb \in \R^{|\SS||\AA|\times|\SS|}$ (we use $|\cdot|$ to denote the cardinality of a set, $\otimes$ the Kronecker product, $\Ib\in\R^{|\SS|\times|\SS|}$ is the identity, and $\oneb \in \R^{|\AA|}$ a vector of ones). 
It is easy to verify that for all $\pb >0$, the solution to~\eqref{eq:LP (value)} is the same and equal to $\vb^*$ (the vector form of $V^*$). 


\paragraph{LP in terms of stationary state-action distribution}
Define the Lagrangian function
\begin{align} \label{eq:Lagrangian}
\LL(\vb, \fb) \coloneqq  \pb^\t \vb + \fb^\t ((1-\gamma)\rb + \gamma \Pb \vb - \Eb \vb)
\end{align}
where  $\fb\geq\zerob \in \R^{|\SS||\AA|}$ is the Lagrangian multiplier. 
By Lagrangian duality, the dual problem of~\eqref{eq:LP (value)} is given as 
$
\max_{\fb\geq \zerob}  \min_{\vb}   \LL(\vb, \fb) $.
Or after substituting the optimality condition of $\vb$
and define $\mub \coloneqq (1-\gamma) \fb$,
we can write the dual problem as another LP problem
\begin{align} \tag{\ref{eq:LP (flow)}}
\max_{\mub \geq \zerob}  \quad   \rb^\t \mub  \qquad
\text{s.t.}\quad  
(1-\gamma)\pb + \gamma \Pb^\t \mub  = \Eb^\t \mub
\end{align}
Note that this problem like~\eqref{eq:LP (value)} is normalized: we have $\norm{\mub}_1 = 1$ because  $\norm{\pb}_1 = 1$, and 
\begin{align*}
\norm{\mub}_1 = \oneb^\t  \Eb^\t \mub = (1-\gamma) \oneb^\t \pb +  \gamma \oneb^\t  \Pb^\t \mub = (1-\gamma) \norm{\pb}_1 + \gamma \norm{\mub}_1
\end{align*}
where we use the facts that $\mub \geq \zerob $ and $\Pb$ is a stochastic transition matrix.
This means that $\mub$ is a valid state-action distribution, from which we see that the equality constraint in~\eqref{eq:LP (flow)} is simply a vector form~\eqref{eq:stationarity (distribution)}. Therefore, \eqref{eq:LP (flow)} is the same as \eqref{eq:RL in stationary distribution} if we define the policy $\pi$ as the conditional distribution based on $\mub$.

\section{Missing Proofs of \cref{sec:revist}} \label{app:missing proofs}

\subsection{Proof of \cref{lm:residual}}
\ResidueEquality*
\begin{proof}
	First note that $F(x,x) = 0$. Then as $x'$ satisfies stationarity, we can use \cref{lm:performance difference lemma} below and write
	\begin{align*}
	r_{ep}(x;x')  &= F(x,x) - F(x, x') \\
	&= 	- F(x, x')\\
	&= - (\pb^\t \vb' - \pb^\t \vb)  -  \mub^\t \ab_{\vb'} +  \mub'^\t \ab_\vb &\since{Definition of $F$ in \eqref{eq:bifunction (RL)}}\\
	&= - \mub' \ab_\vb  -  \mub^\t \ab_{\vb'} +  \mub'^\t \ab_\vb &\since{\cref{lm:performance difference lemma}} \\
	&=  -  \mub^\t \ab_{\vb'} 
	\end{align*}
\end{proof}

\subsection{Proof of~\cref{lm:performance difference lemma}}
\PerformanceDifferenceLemma*
\begin{proof}
	This is the well-known performance difference lemma.
	The proof is based on the stationary properties in \eqref{eq:stationarity (distribution)} and \eqref{eq:statinoarity (value)}, which can be stated in vector form as 
	\begin{align*}
	(\mub^\pi)^\t \Eb \vb^\pi = (\mub^\pi)^\t ((1-\gamma)\rb + \gamma \Pb \vb^\pi) \qquad \text{and} \qquad
	(1-\gamma)\pb + \gamma \Pb^\t \mub^\pi  = \Eb^\t \mub^\pi
	\end{align*}
	The proof is a simple application of these two properties.	
	\begin{align*}
	\pb^\t (\vb^\pi-  \vb') &= \frac{1}{1-\gamma}(\Eb^\t \mub^\pi- \gamma \Pb^\t \mub^\pi)^\t(\vb^\pi-  \vb') \\
	&= \frac{1}{1-\gamma } (\mub^\pi)^\t(  (\Eb - \gamma \Pb)  \vb^\pi-  (\Eb - \gamma \Pb) \vb')\\
	&= \frac{1}{1-\gamma } (\mub^\pi)^\t ( (1-\gamma) \rb -  (\Eb - \gamma \Pb) \vb') = (\mub^\pi)^\t \ab_{\vb'}
	\end{align*}
	where we use the stationarity property of $\mub^\pi$ in the first equality and that  $\vb^\pi$ in the third equality.
\end{proof}

\subsection{Proof of \cref{pr:rough residual bound}}
\RoughResidueBound*
\begin{proof}
This proof mainly follows the steps in \cite{wang2017randomized} but written in our notation.
First \cref{lm:residual} shows $r_{ep}(x;x^*) = -\mub^\t \ab_{\vb^*}$. We then lower bound $-\mub^\t \ab_{\vb^*}$ by reversing the proof of the performance difference lemma (\cref{lm:performance difference lemma}).
\begin{align*}
\mub^\t \ab_{\vb^*} &= \frac{1}{1-\gamma } \mub^\t( (1-\gamma) \rb -  (\Eb - \gamma \Pb) \vb^*) &\since{Definition of $\ab_{\vb^*}$}\\
&= \frac{1}{1-\gamma} \mub^\t(  (\Eb - \gamma \Pb)  \vb^{\pi_\mub}-  (\Eb - \gamma \Pb) \vb^*) &\since{Stationarity of $\vb^{\pi_\mub}$}
\\
&=\frac{1}{1-\gamma} \mub^\t (\Eb- \gamma \Pb)(\vb^{\pi_\mub}-  \vb^*) \\
&=\frac{1}{1-\gamma}\db^\t(\Ib- \gamma \Pb_{\pi_\mub})(\vb^{\pi_\mub}-  \vb^*) 
\end{align*}
where we define $\db \coloneqq \Eb^\t \mub$ and  $\Pb_{\pi_\mub}$ as the state-transition of running policy $\pi_\mub$.

We wish to further upper bound this quantity. 
To proceed, we appeal to the Bellman equation of the optimal value function $\vb^*$ and the stationarity of $\vb^{\pi_\mub}$:
\begin{align*}
\vb^* \geq (1-\gamma)\rb_{\pi_\mub} + \gamma \Pb_{\pi_\mub} \vb^*
\qquad \text{and}\qquad
\vb^{\pi_\mub} = (1-\gamma)\rb_{\pi_\mub} + \gamma \Pb_{\pi_\mub} \vb^{\pi_\mub},
\end{align*}
which together imply that $(\Ib- \gamma \Pb_{\pi_\mub})(\vb^{\pi_\mub}-  \vb^*) \leq 0$.
We will also use the stationarity of $\db^{\pi_\mub}$ (the average state distribution of $\pi_\mub$): 
$
\db^{\pi_\mub} = (1-\gamma) \pb + \gamma \Pb_{\pib_\mub}^\t \db^{\pi_\mub}
$.

Since $\db\geq (1-\gamma) \pb$ in the assumption, we can then write
\begin{align*}
\mub^\t \ab_{\vb^*} 
&= \frac{1}{1-\gamma}\db^\t(\Ib- \gamma \Pb_{\pi_\mub})(\vb^{\pi_\mub}-  \vb^*) \\
&\leq \pb^\t(\Ib- \gamma \Pb_{\pi_\mub})(\vb^{\pi_\mub}-  \vb^*) \\
&\leq - \min_s p(s)  \norm{(\Ib- \gamma \Pb_{\pi_\mub})(\vb^{\pi_\mub}-  \vb^*)}_\infty  \\
&\leq - \min_s p(s) (1-\gamma) \norm{\vb^{\pi_\mub}-  \vb^*}_\infty.
\end{align*}
Finally, flipping the sign of the inequality concludes the proof.
\end{proof}

\subsection{Proof of \cref{pr:residual lower bound}}
\ResidueLowerBound*
\begin{proof}
	
	We show this equality holds for a class of MDPs. 
	For simplicity, let us first consider an MDP with three states $1$, $2$, $3$ and for each state there are three actions ($left$, $right$, $stay$). They correspond to intuitive, deterministic transition dynamics
	\begin{align*}
	\PP(\max\{s-1,1\}|s,left) =1, \quad \PP(\min\{s+1,3\}|s,right) =1, \quad \PP(s|s,stay)=1.
	\end{align*}
	We set the reward as
	$r(s,right)=1$ for $s=1,2,3$ and zero otherwise.
	It is easy to see that the optimal policy is $\pi^*(right|s)=1$, which has value function $\vb^* = [1, 1, 1]^\t$.
	
	Now consider $x= (\vb,\mub)\in\XX$. To define $\mub$, let $\mu(s,a) = d(s)\pi_\mub(a|s)$. We set
	\begin{align*}
	\pi_\mub(right|1) = 1, \quad \pi_\mub(stay|2) = 1, \quad \quad \pi_\mub(right|3) = 1
	\end{align*}
	That is, $\pi_\mub$ is equal to $\pi^*$ except when $s=2$.
	One can verify the value function of this policy is $\vb^{\pi_\mub} = [(1-\gamma), 0, 1]^\t$. 
	
	As far as $d$ is concerned ($\db = \Eb^\t\mub$), suppose the initial distribution is uniform, i.e. $\pb = [1/3,1/3,1/3]^\t$, we choose $d$ as $\db = (1-\gamma)\pb + \gamma [1, 0 , 0]^\t$, which satisfies the assumption in \cref{pr:rough residual bound}. Therefore, we have $\mub \in \MM'$ and we will let $\vb$ be some arbitrary point in $\VV$.
	
	Now we show for this choice $x= (\vb,\mub)\in \VV \times \MM'$, the equality in \cref{pr:rough residual bound} holds. By \cref{lm:residual}, we know $r_{ep}(x;x') = - \mub^\t \ab_{\vb^*}$. Recall the advantage is defined as
	$
	\ab_{\vb^*} = \rb + \frac{1}{1-\gamma}(\gamma \Pb - \Eb) \vb^*
	$.
	Let $A_{V^*}(s,a)$ denote the functional form of $\ab_{\vb^*}$ and 
	define the expected advantage:
	\begin{align*}
	A_{V^*}(s,\pi_\mub) \coloneqq \E_{a\sim\pi_\mub}[A_{V^*}(s,a)].
	\end{align*}
	We can verify it has the following values:
	\begin{align*}
	A_{V^*}(1,\pi_\mub) = 0, \quad A_{V^*}(2,\pi_\mub) = -1, \quad 
	A_{V^*}(3,\pi_\mub) = 0.
	\end{align*}
	
	Thus, the above construction yields
	\begin{align*}
	r_{ep}(x;x^*) = - \mub^\t \ab_{\vb^*} = \frac{(1-\gamma)}{3} = (1-\gamma) \min_s p(s) \norm{\vb^* - \vb^{\pi_\mub}}_\infty 
	\end{align*}
	One can easily generalize this $3$-state MDP to an $|\SS|$-state MDP where states are partitioned into three groups. 
\end{proof}

\section{Missing Proofs of \cref{sec:the reduction}} \label{app:missing proofs of reduction}

\subsection{Proof of \cref{pr:clever residual bound}}
\CleverResidueBound*
\begin{proof}
	First we generalize \cref{lm:residual}.
	\begin{restatable}{lemma}{RelaxedResidueEquality}\label{lm:relaxed residual}
		Let $x = (\vb, \mub)$ be arbitrary. Consider $\tilde{x}' = (\vb'+\ub', \mub')$, where $\vb'$ and $\mub'$ are the value function and state-action distribution of policy $\pi_{\mub'}$, and $\ub'$ is arbitrary. It holds that
		$
		r_{ep}(x;\tilde{x}') = - \mub^\t \ab_{\vb'} - \bb_{\mub}^\t \ub' 
		$.
	\end{restatable}
	To proceed, we write  $y_x^* = (\vb^* + (\vb^{\pi_\mub} - \vb^*), \mub^*)$ and use \cref{lm:relaxed residual}, which gives
	$
	r_{ep}(x;y_x^*) = - \mub^\t \ab_{\vb^*} - \bb_{\mub}^\t (\vb^{\pi_\mub} - \vb^*)
	$. 
	To relate this equality to the policy performance gap, we also need the following equality.
	\begin{restatable}{lemma}{AdvantageBound}\label{lm:advantage bound}
		For $\mub\in\MM$, it holds that $- \mub^\t \ab_{\vb^*}= V^{*}(p) - V^{\pi_\mub}(p)  + \bb_\mub^\t (\vb^{\pi_\mub} - \vb^*)$.
	\end{restatable}
	Together they imply the desired equality
	$r_{ep}(x;y_x^*) = V^{*}(p) - V^{\pi_\mub}(p)$.
\end{proof}

\subsubsection{Proof of \cref{lm:relaxed residual}}
\RelaxedResidueEquality*
\begin{proof}
	Let $x' = (\vb',\mub')$.
	As shorthand, define $\fb'\coloneqq \vb'+\ub'$, and $\Lb \coloneqq \frac{1}{1-\gamma}(\gamma \Pb - \Eb)$ (i.e. we can write $\ab_\fb = \rb + \Lb \fb$). Because $r_{ep}(x;x') = - F(x, x') = - (\pb^\t \vb' +  \mub^\t \ab_{\vb'} - \pb^\t \vb -  \mub'^\t \ab_\vb)$, we can write
	\begin{align*}
	r_{ep}(x; \tilde{x}')
	&= - \pb^\t \fb' - \mub^\t \ab_{\fb'} + \pb^\t \vb +  \mub'^\t \ab_\vb\\
	&= \left(- \pb^\t \vb' - \mub^\t \ab_{\vb'} + \pb^\t \vb +  \mub'^\t \ab_\vb \right) - \pb^\t \ub'  - \mub^\t \Lb \ub' \\
	&= r_{ep}(x;x') - \pb^\t \ub'  - \mub^\t \Lb \ub' \\
	&= r_{ep}(x;x') - \bb_\mub^\t \ub'
	\end{align*}
	Finally, by \cref{lm:residual}, we have also $r_{ep}(x;x')= -  \mub^\t \ab_{\vb'}$ and therefore the final equality.
\end{proof}

\subsubsection{Proof of \cref{lm:advantage bound}}

\AdvantageBound*
\begin{proof}
	Following the setup in \cref{lm:relaxed residual}, we prove the statement by the rearrangement below: 
	\begin{align*}
	- \mub^\t \ab_{\vb'} 
	&= - (\mub^{\pi_\mub})^\t \ab_{\vb'} + (\mub^{\pi_\mub})^\t \ab_{\vb'} - \mub^\t \ab_{\vb'}\\
	&= V^{\pi'}(p) - V^{\pi_\mub}(p) + (\mub^{\pi_\mub} - \mub)^\t \ab_{\vb'}\\ 
	&= \left(V^{\pi'}(p) - V^{\pi_\mub}(p) \right) + (\mub^{\pi_\mub} - \mub)^\t \rb +  (\mub^{\pi_\mub} - \mub)^\t \Lb {\vb'}
	\end{align*}
	where the second equality is due to the performance difference lemma, i.e. \cref{lm:performance difference lemma}, and the last equality uses the definition  $\ab_{\vb'} = \rb + \Lb \vb'$.
	For the second term above, let $\rb_{\pi_{\mub}}$ and $\Pb_{\pi_\mub}$ denote the expected reward and transition under $\pi_{\mub}$. Because $\mub \in \MM$, we can rewrite it as 
	\begin{align*}
	(\mub^{\pi_\mub} - \mub)^\t \rb 
	&= (\Eb^\t\mub^{\pi_\mub} - \Eb^\t\mub) \rb_{\pi_{\mub}}\\
	&= ((1-\gamma) \pb + \gamma \Pb^\t\mub^{\pi_\mub}  - \Eb^\t\mub) \rb_{\pi_{\mub}} \\
	&= (1-\gamma) \bb_\mub^\t \rb_{\pi_{\mub}} + 
	\gamma (\mub^{\pi_\mub}  -\mub)^\t \Pb \rb_{\pi_{\mub}} \\
	&= (1-\gamma) \bb_\mub^\t \left( \rb_{\pi_{\mub}} + \gamma \Pb_{\pi_\mub} \rb_{\pi_{\mub}} + \gamma^2 \Pb_{\pi_\mub}^2 \rb_{\pi_{\mub}} + \dots \right)\\
	&= 
	\bb_\mub^\t \vb^{\pi_\mub}
	\end{align*}
	where the second equality uses the stationarity of $\mub^{\pib_\mub}$ given by \eqref{eq:stationarity (distribution)}.
	For the third term, it can be written 
	\begin{align*}
	(\mub^{\pi_\mub} - \mub)^\t \Lb {\vb'} = (- \pb -  \Lb^\t \mub)^\t {\vb'} 
	= - \bb_\mub^\t\vb'
	\end{align*}
	where the first equality uses stationarity, i.e. $\bb_{\mub^{\pi_\mub}}= \pb + \Lb^\t \mub^{\pi_\mub}=  0$.
	Finally combining the three steps, we have
	\begin{align*}
	- \mub^\t \ab_{\vb'} 
	&= V^{\pi'}(p) - V^{\pi_\mub}(p)  + \bb_\mub(\vb^{\pi_\mub} - \vb') 
	\end{align*}
\end{proof}

\subsection{Proof of \cref{cr:reduction for funcapp}}
\ReductionForFunctionApporximator*
\begin{proof}
	This can be proved by a simple rearrangement
	\begin{align*}
	&V^*(p) - V^{\hat{\pi}_N}(p) 
	=  r_{ep}(\hat{x}_N;y_N^*)= \epsilon_{\Theta,N} + \max_{x_\theta\in\XX_\theta} r_{ep}(\hat{x}_N; x_\theta)\leq \epsilon_{\Theta,N} + \frac{\regret_N(\Theta)}{N}
	\end{align*}
	where the first equality is \cref{pr:clever residual bound} and the last inequality is due to the skew-symmetry of $F$, similar to the proof of \cref{th:reduction of RL}.
\end{proof}

\subsection{Proof of \cref{pr:size of epsilon}}
\SizeOfEpsilon*
\begin{proof}
	For shorthand, let us set $x = (\vb, \mub) = \hat{x}_N$ and write also $\pi_\mub = \hat{\pi}_N$ as the associated policy. 
	Let $y_{x}^* = (\vb^{\pi_\mub}, \mub^*)$ 
	and similarly let $x_\theta = (\vb_\theta, \mub_\theta) \in \XX_\Theta$.
	With $r_{ep}(x;x') = -F(x,x')$ and \eqref{eq:bifunction (RL)}, we can write 
	\begin{align*}
	r_{ep}(x;y_{x}^*) - r_{ep}(x; x_\theta) 
	&=  \left(
	- \pb^\t  \vb^{\pi_\mub} - \mub^\t \ab_{\vb^{\pi_\mub}} + \pb^\t \vb +  {\mub^*}^\t \ab_\vb \right) 
	- \left(
	- \pb^\t \vb_\theta - \mub^\t \ab_{\vb_\theta} + \pb^\t \vb +  \mub_\theta^\t \ab_\vb \right)\\
	&= \pb^\t (\vb_\theta - \vb^{\pi_\mub})  + (\mub^* - \mub_\theta)^\t  \ab_{\vb} + \mub^\t (\ab_{\vb_\theta} - \ab_{\vb^{\pi_\mub}} )\\
	&= \bb_{\mub}^\t (\vb_\theta - \vb^{\pi_\mub})  + (\mub^* - \mub_\theta)^\t  \ab_{\vb} 
	\end{align*}

	Next we quantize the size of $\ab_\vb$ and $\bb_\mub$. 
	\begin{lemma} \label{lm:size of a and b}
	For $(\vb, \mub) \in \XX$, 
	$\norm{\ab_\vb}_\infty \leq \frac{1}{1-\gamma}$ and $\norm{\bb_\mub}_1 \leq \frac{2}{1-\gamma}$.
	\end{lemma}
	\begin{proof}[Proof of \cref{lm:size of a and b}]
  	Let $\Delta$ denote the set of distributions
	\begin{align*}
	\norm{\ab_\vb}_\infty &= \frac{1}{1-\gamma} \norm{(1-\gamma)\rb+\gamma\Pb\vb - \Eb\vb}_\infty \leq  \frac{1}{1-\gamma
	} \max_{a,b\in[0,1]} |a-b|\leq \frac{1}{1-\gamma}
	\\
	\norm{\bb_\mub}_1 &= \frac{1}{1-\gamma} \norm{(1-\gamma)\pb + \gamma \Pb^\t \mub - \Eb^\t\mub}_1 \leq \frac{1}{1-\gamma} \max_{\qb,\qb'\in\Delta}\norm{\qb - \qb'}_1 \leq \frac{2}{1-\gamma}
	\end{align*}
	\end{proof}

	Therefore, we have preliminary upper bounds:
	\begin{align*}
	(\mub^* - \mub_\theta)^\t  \ab_{\vb} &\leq \norm{\ab_{\vb}}_\infty \norm{\mub^* - \mub_\theta}_1 \leq \frac{1}{1-\gamma} \norm{\mub^* - \mub_\theta}_1
	\\
	\bb_{\mub}^\t (\vb_\theta - \vb^{\pi_\mub}) &\leq  \norm{\bb_\mub}_1 \norm{\vb_\theta - \vb^{\pi_\mub}}_\infty \leq \frac{2}{1-\gamma} \norm{\vb_\theta - \vb^{\pi_\mub}}_\infty
	\end{align*}
	However, the second inequality above can be very conservative, especially when $\bb_\mub \approx 0$ which can be likely when it is close to the end of policy optimization. To this end, we introduce a free vector $\wb \geq 1$. Define norms $\norm{\vb}_{\infty, 1/\wb} = \max_i \frac{|v_i|}{w_i} $ and $\norm{\deltab}_{1,\wb} = \sum_{i} w_i |\delta_i|$. Then we can instead have an upper bound 
	\begin{align*}
	\bb_{\mub}^\t (\vb_\theta - \vb^{\pi_\mub}) \leq \min_{\wb:\wb\geq 1} \norm{\bb_\mub}_{1,\wb} \norm{\vb_\theta - \vb^{\pi_\mub}}_{\infty,1/\wb}
	\end{align*}
	Notice that when $\wb = \oneb$ the above inequality reduces to $\bb_{\mub}^\t (\vb_\theta - \vb^{\pi_\mub}) \leq  \norm{\bb_\mub}_1 \norm{\vb_\theta - \vb^{\pi_\mub}}_\infty$, which as we showed has an upper bound $\frac{2}{1-\gamma} \norm{\vb_\theta - \vb^{\pi_\mub}}_\infty$.
	
	Combining the above upper bounds, we have an upper bound on $\epsilon_{\Theta,N}$: 
	\begin{align*}
	\epsilon_{\Theta,N} =  r_{ep}(x;y_{x}^*) - r_{ep}(x; x_\theta)  
	&\leq \frac{1}{1-\gamma} \norm{\mub_\theta - \mub^*}_1 + 
	\min_{\wb:\wb\geq 1} \norm{\bb_\mub}_{1,\wb} \norm{\vb_\theta - \vb^{\pi_\mub}}_{\infty,1/\wb}\\
	&\leq 
	\frac{1}{1-\gamma} \left(  \norm{\mub_\theta - \mub^*}_1
	+ 2 \norm{\vb_\theta - \vb^{\pi_\mub}}_\infty \right).
	\end{align*}
	Since it holds for any $\theta \in \Theta$, we can minimize the right-hand side over all possible choices.
\end{proof}

\section{Proof of Sample Complexity of Mirror Descent} \label{app:proof of sample complexity of mirror descent}
\SampleComplexityOfMirrorDescent*

The proof is a combination of the basic property of mirror descent (\cref{lm:mirror descent}) and the martingale concentration.
Define $K = |\SS||\AA|$ and $\kappa = \frac{1}{1-\gamma}$ as shorthands. We first slightly modify the per-round loss used to compute the gradient. 
Recall $
l_n(x) \coloneqq \pb^\t \vb +  \mub_n^\t \ab_{\vb} - \pb^\t \vb_n -  \mub^\t \ab_{\vb_n}
$ and let us consider instead a loss function
\begin{align*}
h_n(x) \coloneqq \bb_{\mub_n}^\t \vb  + \mub^\t (\kappa \oneb - \ab_{\vb_n})
\end{align*}
which shifts $l_n$ by a constant in each round. (Note for all the decisions $(\vb_n, \mub_n)$ produced by \cref{alg:md for RL} $\mub_n$ always satisfies $\norm{\mub_n}_1 = 1$).
One can verify that $l_n(x)-l_n(x') = h_n(x) - h_n(x')$, for all $x,x'\in \XX$, when $\mub, \mub'$ in $x$ and $x'$ satisfy $\norm{\mub}_1 = \norm{\mub'}_1$ (which holds for \cref{alg:md for RL}). As the definition of regret is relative, we may work with $h_n$ in online learning and use it to define the feedback. 

The reason for using $h_n$ instead of $l_n$ is to make $\nabla_\mub h_n((\vb,\mub))$ (and its unbiased approximation) a positive vector (because $\kappa \geq \norm{\ab_{\vb}}_\infty$ for any $\vb \in \VV$), so that the regret bound can have a better dependency on the dimension for learning $\mub$ that lives in the simplex $\MM$. This is a common trick used in the online learning, e.g. in EXP3.
 
To run mirror descent, we set the first-order feedback $g_n$ received by the learner as an unbiased estimate of $\nabla h_n(x_n)$. For round $n$, we construct $g_n$ based on \emph{two} calls of the generative model:
\begin{align*}
g_n 
= \begin{bmatrix}
\gb_{n,v}\\
\gb_{n,\mu}
\end{bmatrix}
= \begin{bmatrix}
\tilde{\pb}_n + \frac{1}{1-\gamma}(\gamma\tilde{\Pb}_n -\Eb_n)^\t \tilde{\mub}_n \\
K (\kappa \hat{\oneb}_n - \hat{\rb}_n - \frac{1}{1-\gamma}(\gamma\hat{\Pb}_n-\hat{\Eb}_n)\vb_n)
\end{bmatrix}
\end{align*}
For $\gb_{n,v}$, we sample $\pb$,  then sample $\mub_n$ to get a state-action pair, and finally query the transition dynamics $\Pb$ at the state-action pair sampled from $\mub_n$. ($\tilde{\pb}_n$, $\tilde{\Pb}_n$, and $\tilde{\mub}_n$ denote the single-sample empirical approximation of these probabilities.)
For $\gb_{n,\mu}$, we first sample \emph{uniformly} a state-action pair (which explains the factor $K$), and then query the reward $\rb$ and the transition dynamics $\Pb$. ($\hat{\oneb}_n$, $\hat{\rb}_n$, $\hat{\Pb}_n$, and $\hat{\Eb}_n$ denote the single-sample empirical estimates.) To emphasize, we use $\tilde{}$ and $\hat{}$ to distinguish the empirical quantities obtained by these  two independent queries. By construction, we have $\gb_{n,\mu} \geq 0$.
It is clear that this direction $g_n$ is unbiased, i.e. $\E[g_n] = \nabla h_n(x_n)$. Moreover, it is extremely sparse and can be computed using $O(1)$ sample, computational, and memory complexities. 

Let $y_N^* = (\vb^{\hat{\pi}_N}, \mub^*)$.
We bound the regret by the following rearrangement. 
\begin{align} \label{eq:regret rearrangement}
\regret_N (y_N^*)
&= \sum_{n=1}^{N} l_n(x_n) - \sum_{n=1}^N l_n(y_N^*) \nonumber \\
&= \sum_{n=1}^{N} h_n(x_n) - \sum_{n=1}^N h_n(y_N^* ) \nonumber \\
&= \sum_{n=1}^{N} \nabla h_n(x_n)^\t (x_n-y_N^*) \nonumber \\
&= \left(\sum_{n=1}^{N} (\nabla h_n(x_n)-g_n)^\t x_n\right)  + \left(\sum_{n=1}^N g_n^\t (x_n - y_N^*) \right) + \left(\sum_{n=1}^N   (g_n - \nabla h_n(x_n))^\t y_N^* \right) \nonumber\\
&\leq 
\left(\sum_{n=1}^{N} (\nabla h_n(x_n)-g_n)^\t x_n\right)  
+ \left( \max_{x\in\XX} \sum_{n=1}^N g_n^\t (x_n -x) \right)+ \left(\sum_{n=1}^N   (g_n - \nabla h_n(x_n))^\t y_N^* \right),
\end{align}

where the third equality comes from $h_n$ being linear. We recognize the first term is a martingale $M_N = \sum_{n=1}^{N} (\nabla h_n(x_n)-g_n)^\t x_n$, because $x_n$ does not depend on $g_n$. Therefore, we can appeal to standard martingale concentration property. For the second term, it can be upper bounded by standard regret analysis of mirror descent, by treating $g_n^\t x$ as the per-round loss. For the third term, because $y_N^* = (\vb^{\hat{\pi}_N}, \mub^*)$ depends on $\{g_n\}_{n=1}^N$, it is not a martingale. Nonetheless, we will be able to handle it through a union bound.
%
Below, we give details for bounding these three terms.

\subsection{The First Term: Martingale Concentration} \label{app:martingale concentration}

For the first term, $\sum_{n=1}^{N} (\nabla h_n(x_n)-g_n)^\t x_n$, we use a martingale concentration property.
Specifically, we adopt a Bernstein-type inequality~\citep[Theorem 3.15]{mcdiarmid1998concentration}:
\begin{lemma} \label{lm:martingale (Bernstein)}
{\normalfont\citep[Theorem 3.15]{mcdiarmid1998concentration} } 
Let $M_0, \dots, M_N$ be a martingale and let $F_0\subseteq F_1 \subseteq \dots \subseteq F_n$ be the filtration such that $M_n = \E_{|F_{n}}[M_{n+1}]$. Suppose there are $b, \sigma < \infty$ such that for all $n$, given $F_{n-1}$, $M_n - M_{n-1} \leq b$, and $\Vbb_{|F_{n-1}}[M_n - M_{n-1}] \leq \sigma^2$  almost surely. Then for any $\epsilon\geq0$,
\begin{align*}
P(M_N - M_0 \geq \epsilon) \leq \exp\left(\frac{- \epsilon^2}{2 N \sigma^2 (1+\frac{b\epsilon}{3N\sigma^2})}   \right).
\end{align*}
\end{lemma}
\cref{lm:martingale (Bernstein)} implies, with probability at least $1-\delta$, 
\begin{align*}
M_N - M_0 \leq \sqrt{2 N \sigma^2 (1+ o(1)) \log\left(\frac{1}{\delta}\right)},
\end{align*}
where $o(1)$ means convergence to $0$ as $N\to\infty$.

To apply \cref{lm:martingale (Bernstein)}, we need to provide bounds on the properties of the martingale difference:
\begin{align*}
M_n - M_{n-1} 
&= (\nabla h_n(x_n)-g_n)^\t x_n \\
&= (\kappa \oneb - \ab_{\vb_n} - \gb_{n,\mu})^\t \mub_n + (\bb_{\mub_n}-\gb_{n,v})^\t \vb_n.
\end{align*}

For the first term $(\kappa \oneb - \ab_{\vb_n} - \gb_{n,\mu})^\t \mub_n$, we use the lemma below:
\begin{lemma} \label{lm:martingale difference for mu}
Let $\mub \in \MM$ be arbitrary, chosen independently from the randomness of $\gb_{n,\mu}$ when $F_{n-1}$ is given. Then it holds
$ |(\kappa \oneb - \ab_{\vb_n} - \gb_{n,\mu})^\t \mub|  \leq \frac{2(1+K)}{1-\gamma}$ 
and $\Vbb_{|F_{n-1}}[(\kappa \oneb - \ab_{\vb_n} - \gb_{n,\mu})^\t \mub] \leq \frac{4K}{(1-\gamma)^2}$.
\end{lemma}
\begin{proof}
By triangular inequality, 
\begin{align*}
 |(\kappa \oneb - \ab_{\vb_n} - \gb_{n,\mu})^\t \mub|  \leq  |(\kappa \oneb - \ab_{\vb_n} )^\t \mub |+ | \gb_{n,\mu}^\t \mub|.
\end{align*}
For the deterministic part, using \cref{lm:size of a and b} and H\"older's inequality,
\begin{align*}
|(\kappa \oneb - \ab_{\vb_n})^\t \mub| \leq \kappa + \norm{\ab_{\vb_n}}_\infty\norm{\mub}_1 \leq \frac{2}{1-\gamma}.
\end{align*}
For the stochastic part, let $i_n$ be index of the sampled state-action pair and $j_n$ be the index of the transited state sampled at the pair given by $i_n$. With abuse of notation, we will use $i_n$ to index both $\SS\times\AA$ and $\SS$. 
With this notation, we may derive
\begin{align*}
|\gb_{n,\mu}^\t \mub |
&= |K \mub^\t (\kappa \hat{\oneb}_n - \hat{\rb}_n - \frac{1}{1-\gamma}(\gamma\hat{\Pb}_n-\hat{\Eb}_n)\vb_n)|\\
&=K \mu_{i_n}|\kappa  - r_{i_n} - \frac{\gamma v_{n,j_n} - v_{n,i_n}}{1-\gamma} |\\
&\leq \frac{2 K \mu_{i_n}}{1-\gamma} \leq  \frac{2 K}{1-\gamma}
\end{align*}
where we use the facts that $r_{i_n},  v_{n,j_n}, v_{n,i_n} \in [0,1]$ and $\mu_{i_n} \leq 1$.

For $\Vbb_{|F_{n-1}}[(\kappa \oneb - \ab_{\vb_n} - \gb_{n,\mu})^\t \mub_n]$, we can write it as 
\begin{align*}
\Vbb_{|F_{n-1}}[(\kappa \oneb - \ab_{\vb_n} - \gb_{n,\mu})^\t \mub] 
&= \Vbb_{|F_{n-1}}[ \gb_{n,\mu}^\t \mub] \\
&\leq \E_{|F_{n-1}}[ |\gb_{n,\mu}^\t \mub_n|^2] \\
&= \sum_{i_n} \frac{1}{K}\E_{j_n|i_n} \left[ K^2 \mu_{i_n}^2 \left(\kappa  - r_{i_n} - \frac{\gamma v_{n,j_n} - v_{n,i_n}}{1-\gamma}\right)^2 \right]\\
&\leq  \frac{4 K }{(1-\gamma)^2} \sum_{i_n} \mu_{i_n}^2  \\
&\leq  \frac{4 K }{(1-\gamma)^2} \left( \sum_{i_n} \mu_{i_n}\right)^2   \leq \frac{4 K }{(1-\gamma)^2}
\end{align*}
where in the second inequality we use the fact that $|\kappa  - r_{i_n} - \frac{\gamma v_{n,j_n} - v_{n,i_n}}{1-\gamma} | \leq \frac{2}{1-\gamma}$.
\end{proof}

For the second term $(\bb_{\mub_n}-\gb_{n,v})^\t \vb_n$, we use the following lemma.
\begin{lemma} \label{lm:martingale difference for v}
Let $\vb \in \VV$ be arbitrary, chosen independently from the randomness of $\gb_{n,v}$ when $F_{n-1}$ is given.. Then it holds that
$ |(\bb_{\mub_n}-\gb_{n,v})^\t \vb|  \leq \frac{4}{1-\gamma}$ 
and $\Vbb_{|F_{n-1}}[(\bb_{\mub_n}-\gb_{n,v})^\t \vb] \leq \frac{4}{(1-\gamma)^2}$.
\end{lemma}
\begin{proof}
We appeal to \cref{lm:size of a and b}, which shows $\norm{\bb_{\mub_n}}_1,\norm{\gb_{n,v}}_1 \leq \frac{2}{1-\gamma}$, and derive
\begin{align*}
|(\bb_{\mub_n}-\gb_{n,v})^\t \vb| \leq (\norm{\bb_{\mub_n}}_1 + \norm{\gb_{n,v}}_1 )\norm{\vb}_\infty \leq \frac{4}{1-\gamma}.
\end{align*}
Similarly, for the variance, we can write
\begin{align*}
\Vbb_{|F_{n-1}}[(\bb_{\mub_n}-\gb_{n,v})^\t \vb]
&= \Vbb_{|F_{n-1}}[\gb_{n,v}^\t \vb]\leq \Ebb_{|F_{n-1}}[(\gb_{n,v}^\t \vb)^2] \leq \frac{4}{(1-\gamma)^2}. \qedhere
\end{align*}
\end{proof}

Thus, with helps from the two lemmas above, we are able to show 
\begin{align*}
M_n - M_{n-1} \leq |(\kappa \oneb - \ab_{\vb_n} - \gb_{n,\mu})^\t \mub_n| + |(\bb_{\mub_n}-\gb_{n,v})^\t \vb_n |\leq  \frac{4+2(1+K)}{1-\gamma}
\end{align*}
as well as (because $\gb_{n,\mu}$ and $\gb_{n,b}$ are computed using independent samples)
\begin{align*}
\Vbb_{|F_{n-1}}[M_n - M_{n-1}]
&\leq \E_{|F_{n-1}}[|(\kappa \oneb - \ab_{\vb_n} - \gb_{n,\mu})^\t \mub_n|^2] + \E_{|F_{n-1}}[|(\bb_{\mub_n}-\gb_{n,v})^\t \vb_n|^2]
\leq \frac{4(1+K)}{(1-\gamma)^2}
\end{align*}
Now, since $M_0 = 0$,  by martingale concentration in \cref{lm:martingale (Bernstein)}, we have
\begin{align*}
P\left( \sum_{n=1}^{N} (\nabla h_n(x_n)-g_n)^\t x_n > \epsilon \right) \leq \exp\left(\frac{- \epsilon^2}{2 N \sigma^2 (1+\frac{b\epsilon}{3N\sigma^2})}   \right)
\end{align*}
with $b = \frac{6+2K}{1-\gamma}$ and $\sigma^2 = \frac{4(1+K)}{(1-\gamma)^2}$. This implies that, with probability at least $1-\delta$, it holds
\begin{align*}
\sum_{n=1}^{N} (\nabla h_n(x_n)-g_n)^\t x_n \leq \sqrt{ N\frac{8(1+K)}{(1-\gamma)^2} (1+o(1))\log\left(\frac{1}{\delta}\right)}
= \tilde{O}\left( \frac{\sqrt{NK\log(\frac{1}{\delta})}}{1-\gamma} \right)
\end{align*}

\subsection{Static Regret of Mirror Descent} \label{app:static regret of mirror descent}

Next we move onto deriving the regret bound of mirror descent  with respect to the online loss sequence:
\begin{align*}
 \max_{x\in\XX} \sum_{n=1}^N g_n^\t (x_n -x) 
\end{align*}
This part is quite standard; nonetheless, we provide complete derivations below.

We first recall a basic property of mirror descent
\begin{lemma} \label{lm:mirror descent}
	Let $\XX$ be a convex set. Suppose $R$ is 1-strongly convex with respect to some norm $\norm{\cdot}$. Let $g$ be an arbitrary vector and define, for $x\in\XX$,
	\begin{align*}
	y =  \argmin_{x' \in \XX} \lr{g}{x'} + B_{R}(x'||x)
	\end{align*}
	Then for all $z \in \XX$,
	\begin{align} \label{eq:mirror descent new decision}
 \lr{g}{y - z} &\leq  B_R(z||x) - B_R(z||y) - B_R(y||x)  
	\end{align}
\end{lemma}
\begin{proof}
	Recall the definition $B_R(x'||x) =  R(x') - R(x) - \lr{\nabla R(x)}{x'-x}$. 
	The optimality of the proximal map can be written as 
	\begin{align*}
	\lr{  g + \nabla R(y) -  \nabla R(x)  }{y - z} \leq 0,  \qquad \forall z \in \XX.
	\end{align*}
	By rearranging the terms, we can rewrite the above inequality in terms of Bregman divergences as follows and derive the first inequality~\eqref{eq:mirror descent new decision}:
	\begin{align*}
	\lr{   g   }{y - z} &\leq  \lr{ \nabla R(x) -  \nabla R(y)  }{y - z} \\
	&=  B_R(z||x) - B_R(z||y) + \lr{ \nabla R(x) -  \nabla R(y)  }{y} -  \lr{\nabla R(x)}{x} + \lr{\nabla R(y)}{y}  + R(x)  - R(y)  \\
	&=  B_R(z||x) - B_R(z||y) + \lr{ \nabla R(x)  }{y - x}   + R(x)  - R(y)  \\
	&=  B_R(z||x) - B_R(z||y) - B_R(y||x),
	\end{align*}
	which concludes the lemma.
\end{proof}

Let $x'\in\XX$ be arbitrary.
Applying this lemma to the $n$th iteration of mirror descent in \eqref{eq:mirror descent}, we get
\begin{align*} 
\lr{g_n}{x_{n+1} - x'} &\leq \frac{1}{\eta} \left( B_R(x'||x_n) - B_R(x'||x_{n+1}) - B_R(x_{n+1}||x_n)  \right)
\end{align*}
By a telescoping sum, we then have
\begin{align*}
\sum_{n=1}^{N} \lr{g_n}{x_{n} - x'} &\leq \frac{1}{\eta} B_R(x'||x_1) + \sum_{n=1}^{N} \lr{g_n}{x_{n+1} - x_n} - \frac{1}{\eta}  B_R(x_{n+1}||x_n).
\end{align*}

We bound the right-hand side as follows.
Recall that based on the geometry of $\XX = \VV \times \MM$, we considered a natural Bregman divergence of the form:
\begin{align*}
B_R(x'||x) =  \frac{1}{2|\SS|} \norm{\vb' - \vb}_2^2 + KL(\mub'||\mub) 
\end{align*}
Let $x_1 = (\vb_1, \mub_1)$ where $\mub_1$ is uniform.
By this choice, we have:
\begin{align*}
\frac{1}{\eta}  B_R(x'||x_1) \leq \frac{1}{\eta}  \max_{x\in\XX}B_R(x||x_1) \leq \frac{1}{\eta}\left( \frac{1}{2} + \log(K) \right).
\end{align*}

We now decompose each item in the above sum as:
\begin{align*}
\lr{g_n}{x_{n+1} - x_n} - \frac{1}{\eta}  B_R(x_{n+1}||x_n)
=& \left( \gb_{n,v}^\t (\vb_{n+1} - \vb_n) - \frac{1}{2\eta  |\SS|} \norm{\vb_n - \vb_{n+1}}_2^2 \right) \\
&+ \left( \gb_{n,\mu}^\t(\mub_{n+1} - \mub_n) - \frac{1}{\eta}KL(\mub_{n+1}||\mub_n)  \right) 
\end{align*} 
and we upper bound them using the two lemmas below (recall $\gb_{n,\mu} \geq 0 $ due to the added $\kappa \oneb$ term).
\begin{restatable}{lemma}{GradientDescentBound} \label{lm:gradient descent bound}
For any vector $x, y, g$ and scalar $\eta >0$, it holds
$\lr{g}{x-y} - \frac{1}{2\eta}\norm{x-y}_2^2 \leq \frac{\eta \norm{g}_2^2}{2}$.
\end{restatable}
\begin{proof}
By Cauchy-Swartz inequality, 
$
\lr{g}{x-y} - \frac{1}{2\eta}\norm{x-y}_2^2 \leq \norm{g}_2\norm{x - y}_2 - \frac{1}{2\eta}\norm{x-y}_2^2 \leq \frac{ \eta \norm{g}_2^2}{2}
$.
\end{proof}

\begin{restatable}{lemma}{EXPBound}
\label{lm:exp3 bound}
Suppose $B_R(x||y) = KL(x||y)$ and $x,y$ are probability distributions, and $g \geq 0$ element-wise. Then, for $\eta >  0$,
\begin{align*}
-\frac{1}{\eta} B_R(y||x)  + \lr{ g }{x - y}
\leq  \frac{\eta}{2}  \sum_i x_i( g_i)^2  = \frac{\eta}{2} \norm{g}_{x}^2.
\end{align*}
\end{restatable}
\begin{proof}
Let $\Delta$ denotes the unit simplex.
\begin{align*}
- B_R(y||x)  + \lr{ \eta  g }{x - y} 
&\leq \lr{ \eta  g }{x }  + \max_{y' \in \Delta}  \lr{ - \eta  g }{y} - B_R(y'||x)\\
&=  \lr{ \eta  g }{x } + \log\left( \sum_i x_i \exp(-\eta g_i)   \right) &\since{convex conjugate of KL divergence}\\
&\leq  \lr{ \eta  g }{x } + \log\left( \sum_i x_i \left( 1 - \eta g_i + \frac{1}{2} (\eta g_i)^2 \right)  \right) &\since{$e^{x} \leq 1 + x + \frac{1}{2}x^2$ for $x\leq 0$}\\
&=  \lr{ \eta  g }{x } + \log\left( 1+ \sum_i x_i \left( -\eta g_i + \frac{1}{2} (\eta g_i)^2 \right)  \right) \\
&\leq   \lr{ \eta  g }{x } + \sum_i x_i \left( -\eta g_i + \frac{1}{2} (\eta g_i)^2 \right)  &\since{$\log(x) \leq x -1$} \\
&=  \frac{1}{2}  \sum_i x_i(\eta g_i)^2  = \frac{\eta^2}{2} \norm{g}_{x}^2.
\end{align*}
Finally, dividing both sides by $\eta$, we get the desired result.
\end{proof}
Thus, we have the upper bound 
$
\lr{g_n}{x_{n+1} - x_n} -\frac{1}{\eta}   B_R(x_{n+1}||x_n)
= \frac{\eta|\SS| \norm{\gb_{n,v}}_2^2}{2} + \frac{\eta \norm{\gb_{n,\mu}}_{\mub_n}^2}{2}
$. Together with the upper bound on $\frac{1}{\eta}  B_R(x'||x_1)$, it implies that
\begin{align} \label{eq:initial telesum of md}
\sum_{n=1}^{N} \lr{g_n}{x_{n} - x'} &\leq  \frac{1}{\eta}  B_R(x'||x_1) + \sum_{n=1}^{N} \lr{g_n}{x_{n+1} - x_n} - \frac{1}{\eta} B_R(x_{n+1}||x_n)  \nonumber \\
&\leq \frac{1}{\eta}\left( \frac{1}{2} + \log(K) \right) + \frac{\eta}{2} \sum_{n=1}^{N} |\SS|\norm{\gb_{n,v}}_2^2 + \norm{\gb_{n,\mu}}_{\mub_n}^2.
\end{align}

We can expect, with high probability, $\sum_{n=1}^{N} |\SS|\norm{\gb_{n,v}}_2^2 + \norm{\gb_{n,\mu}}_{\mub_n}^2$ concentrates toward its expectation, i.e.
\begin{align*}
\sum_{n=1}^{N} |\SS| \norm{\gb_{n,v}}_2^2 + \norm{\gb_{n,\mu}}_{\mub_n}^2 \leq \sum_{n=1}^{N} \E[|\SS| \norm{\gb_{n,v}}_2^2 + \norm{\gb_{n,\mu}}_{\mub_n}^2] + o(N).
\end{align*}
Below we quantify this relationship using martingale concentration. First we bound the expectation.
\begin{lemma} \label{lm:size of sampled grad}
$
\E[\norm{\gb_{n,v}}_2^2] \leq \frac{4}{(1-\gamma)^2} $ and $\E[\norm{\gb_{n,\mu}}_{\mub_n}^2] \leq \frac{4K}{(1-\gamma)^2}$.
\end{lemma}
\begin{proof}
For the first statement, using the fact that $\norm{\cdot}_2 \leq \norm{\cdot}_1$ and \cref{lm:size of a and b}, we can write 
\begin{align*}
\E[\norm{\gb_{n,v}}_2^2] 
\leq \E[\norm{\gb_{n,v}}_1^2] 
= \E[\norm{\tilde{\pb}_n + \frac{1}{1-\gamma}(\gamma\tilde{\Pb}_n -\Eb_n)^\t \tilde{\mub}_n }_1^2] \leq \frac{4}{(1-\gamma)^2}.
\end{align*}

For the second statement, let $i_n$ be the index of the sampled state-action pair and $j_n$ be the index of the transited-to state sampled at the pair given by $i_n$. With abuse of notation, we will use $i_n$ to index both $\SS\times\AA$ and $\SS$. 
\begin{align*}
\E [ \norm{\gb_{n,\mu}}_{\mub_n}^2] 
&= \E\left[ \sum_{i_n} \frac{1}{K}\E_{j_n|i_n} \left[ K^2 \mu_{i_n} \left(\kappa  - r_{i_n} - \frac{\gamma v_{n,j_n} - v_{n,i_n}}{1-\gamma}\right)^2 \right] \right]\\
&\leq  \frac{4 K }{(1-\gamma)^2} \E\left[ \sum_{i_n} \mu_{i_n} \right]  \leq \frac{4 K }{(1-\gamma)^2}. \qedhere
\end{align*}
\end{proof}

To bound the tail, we resort to the H\"offding-Azuma inequality of martingale~\citep[Theorem 3.14]{mcdiarmid1998concentration}. 
\begin{lemma}[Azuma-Hoeffding] \label{lm:martgingale (Hoeffding)}
Let $M_0, \dots, M_N$ be a martingale and let $F_0\subseteq F_1 \subseteq \dots \subseteq F_n$ be the filtration such that $M_n = \E_{|F_{n}}[M_{n+1}]$. Suppose there exists $b < \infty$ such that for all $n$, given $F_{n-1}$, $|M_n - M_{n-1}|\leq b$. Then for any $\epsilon\geq0$,	
\begin{align*}
P(M_N - M_0 \geq \epsilon) \leq \exp\left(\frac{-2\epsilon^2}{N b^2 }  \right)
\end{align*}
\end{lemma}
To apply \cref{lm:martgingale (Hoeffding)}, we consider the martingale 
\begin{align*}
M_N = \sum_{n=1}^{N}|\SS| \norm{\gb_{n,v}}_2^2 + \norm{\gb_{n,\mu}}_{\mub_n}^2 - \left(\sum_{n=1}^{N}\E[|\SS| \norm{\gb_{n,v}}_2^2 + \norm{\gb_{n,\mu}}_{\mub_n}^2] \right)
\end{align*}

To bound the change of the size of martingale difference $|M_n - M_{n-1}|$, we follow similar steps to \cref{lm:size of sampled grad}. 
\begin{lemma}
$\norm{\gb_{n,v}}_2^2 \leq \frac{4}{(1-\gamma)^2}$ and $\norm{\gb_{n,\mu}}_{\mub_n}^2 \leq \frac{4K^2}{(1-\gamma)^2}$.
\end{lemma}
Note $\norm{\gb_{n,\mu}}_\mub^2$ is $K$-factor larger than $\E[\norm{\gb_{n,\mu}}_\mub^2]$) and $K\geq1$.
Therefore, we have
\begin{align*}
|M_n - M_{n-1}| \leq|\SS| \norm{\gb_{n,v}}_2^2 + \norm{\gb_{n,\mu}}_{\mub_n}^2 +|\SS| \E[\norm{\gb_{n,v}}_2^2] +\E[ \norm{\gb_{n,\mu}}_{\mub_n}^2] \leq \frac{8(|\SS|+K^2)}{(1-\gamma)^2}
\end{align*}
Combining these results, we have, with probability as least $1-\delta$, 
\begin{align*}
\sum_{n=1}^{N} |\SS|\norm{\gb_{n,v}}_2^2 +  \norm{\gb_{n,\mu}}_{\mub_n}^2 
&\leq \sum_{n=1}^{N}  \E[|\SS|\norm{\gb_{n,v}}_2^2 + \norm{\gb_{n,\mu}}_{\mub_n}^2] + \frac{4\sqrt{2}(|\SS|+K^2)}{(1-\gamma)^2} \sqrt{N \log\left(\frac{1}{\delta}\right)}\\
&\leq  \frac{4(K+|\SS|)}{(1-\gamma)^2} N + \frac{4\sqrt{2}(|\SS|+K^2)}{(1-\gamma)^2} \sqrt{N \log\left(\frac{1}{\delta}\right)}
\end{align*}

Now we suppose we set $\eta = \frac{1-\gamma}{\sqrt{KN}}$. From \eqref{eq:initial telesum of md}, we then have 
\begin{align*}
\sum_{n=1}^{N} \lr{g_n}{x_{n} - x'} 
&\leq \frac{1}{\eta}\left( \frac{1}{2} + \log(K) \right) + \frac{\eta}{2} \sum_{n=1}^{N}|\SS| \norm{\gb_{n,v}}_2^2 + \norm{\gb_{n,\mu}}_{\mub_n}^2 \\
&\leq  \frac{\sqrt{KN}}{1-\gamma} \left( \frac{1}{2} + \log(K) \right)
+
\frac{1-\gamma}{\sqrt{KN}} \left(  \frac{2(K+|\SS|)}{(1-\gamma)^2} N + \frac{2\sqrt{2}(|\SS|+K^2)}{(1-\gamma)^2} \sqrt{N \log\left(\frac{1}{\delta}\right)}\right)\\
&\leq  \tilde{O}\left(\frac{\sqrt{KN}}{1-\gamma} + \frac{\sqrt{K^3 \log\frac{1}{\delta}}}{1-\gamma} \right).
\end{align*}


\subsection{Union Bound} \label{app:union bound}

Lastly, we provide an upper bound on the last component:
\begin{align*}
\sum_{n=1}^N   (g_n - \nabla h_n(x_n))^\t y_N^*.
\end{align*}
Because $y_N^*$ depends on $g_n$, this term does not obey martingale concentration like the first component $\sum_{n=1}^{N} (\nabla h_n(x_n)-g_n)^\t x_n$ which we analyzed in \cref{app:martingale concentration} To resolve this issue, we utilize the concept of covering number and derive a union bound.

Recall for a compact set $\ZZ$ in a norm space, the covering number $\NN(\ZZ,\epsilon)$ with $\epsilon>0$ is the minimal number of $\epsilon$-balls that covers $\ZZ$. That is, there is a set $ \{ z_i \in \ZZ \}_{i=1}^{\NN(\ZZ,\epsilon)}$ such that $\max_{z\in\ZZ} \min_{z'\in B(\ZZ,\epsilon)} \norm{z-z'} \leq \epsilon$. Usually the covering number $\NN(\ZZ,\epsilon)$ is polynomial in $\epsilon$ and perhaps exponential in the ambient dimension of $\ZZ$.

The idea of covering number can be used to provide a union bound of concentration over compact sets, which we summarize as a lemma below.
\begin{lemma} \label{lm:union bound of Lipshitz functions}
Let $f,g$ be two random $L$-Lipschitz functions. Suppose for some $a>0$ and some fixed $z\in\ZZ$  selected independently of $f,g$, it holds 
\begin{align*}
P\left( \left| f(z) - g(z) \right| > \epsilon \right) \leq \exp\left( -a\epsilon^2\right)
\end{align*}
Then it holds that
\begin{align*}
P\left( \sup_{z\in\ZZ} \left| f(z) - g(z) \right| > \epsilon \right) \leq \NN\left(\ZZ,\frac{\epsilon}{4L}\right)  \exp \left( \frac{-a\epsilon^2}{4}\right)
\end{align*}
\end{lemma}
\begin{proof}
Let $\CC$ denote a set of covers of size $\NN(\ZZ,\frac{\epsilon}{4L})$
Then, for any $z \in \ZZ$ which could depend on $f,g$,
\begin{align*}
\left| f(z) - g(z) \right| 
&\leq \min_{z' \in \CC} \left| f(z) - f(z') \right| + \left| f(z') - g(z') \right| + \left| g(z') - g(z) \right|\\
&\leq \min_{z' \in \CC} 2L\norm{z-z'}+ \left| f(z') - g(z')] \right| \\
&\leq  \frac{\epsilon}{2} + \max_{z'\in \CC}\left| f(z') - g(z') \right| 
\end{align*}
Thus, 
$
 \sup_{z\in\ZZ} \left| f(z) - g(z) \right| > \epsilon \implies \max_{z'\in \CC}\left| f(z') - g(z') \right|  > \frac{\epsilon}{2}
$. Therefore, we have the union bound.
\begin{align*}
P\left( \sup_{z\in\ZZ} \left| f(z) - \E[f(z)] \right| > \epsilon \right) &\leq \NN \left(\ZZ,\frac{\epsilon}{4L}\right)  \exp \left( \frac{-a\epsilon^2}{4}\right). \qedhere
\end{align*}
\end{proof}

We now use \cref{lm:union bound of Lipshitz functions} to bound the component $\sum_{n=1}^N   (g_n - \nabla h_n(x_n))^\t y_N^*$. 
We recall by definition, for $x = (\vb, \mub)$, 
\begin{align*}
(\nabla h_n(x_n) - g_n)^\t x = (\kappa \oneb - \ab_{\vb_n} - \gb_{n,\mu})^\t \mub + (\bb_{\mub_n}-\gb_{n,v})^\t \vb
\end{align*}
Because  $y_N^* = (\vb^{\hat{\pi}_N}, \mub^*)$, we can write the sum of interest as
\begin{align*}
\sum_{n=1}^N   (g_n - \nabla h_n(x_n))^\t y_N^* 
&= \sum_{n=1}^N  (\gb_{n,\mu} - \kappa \oneb + \ab_{\vb_n})^\t \mub^* + \sum_{n=1}^N  (\gb_{n,v}- \bb_{\mub_n})^\t \vb^{\hat{\pi}_N}
\end{align*}
For the first term, because $\mub^*$ is set beforehand by the MDP definition and does not depend on the randomness during learning, it is a martingale and we can apply the steps in \cref{app:martingale concentration} to show,
\begin{align*}
\sum_{n=1}^N  (\gb_{n,\mu} - \kappa \oneb + \ab_{\vb_n})^\t \mub^* =\tilde{O}\left( \frac{\sqrt{NK\log(\frac{1}{\delta})}}{1-\gamma} \right) 
\end{align*}
For the second term, because $\vb^{\hat{\pi}_N}$ depends on the randomness in the learning process, we need to use a union bound. Following the steps in \cref{app:martingale concentration}, we see that for some fixed $\vb \in \VV$, it holds
\begin{align*}
P\left( \left| \sum_{n=1}^N  (\gb_{n,v}- \bb_{\mub_n})^\t \vb \right| > \epsilon \right) \leq \exp\left( - \frac{(1 - \gamma)^2}{N}\epsilon^2 \right)
\end{align*}
where some constants were ignored for the sake of conciseness. Note also that it does not have the $\sqrt{K}$ factor because of 
\cref{lm:martingale difference for v}. To apply \cref{lm:union bound of Lipshitz functions}, we need to know the order of covering number of $\VV$. Since $\VV$ is an $|\SS|$-dimensional unit cube in the positive orthant, it is straightforward to show
$
\NN\left(\VV,\epsilon\right) \leq \max(1, (1/\epsilon)^{|\SS|})
$ (by simply discretizing evenly in each dimension). Moreover, the functions $\sum_{n=1}^N \gb_{n,v}^\t \vb$ and $\sum_{n=1}^N  \bb_{\mub_n}^\t \vb$ are $\frac{N}{1-\gamma}$-Lipschitz in $\norm{\cdot}_\infty$.

Applying \cref{lm:union bound of Lipshitz functions} then gives us:
\begin{align*}
P\left( \sup_{v\in\VV} \left| \sum_{n=1}^N  (\gb_{n,v}- \bb_{\mub_n})^\t \vb \right| > \epsilon \right) \leq \NN \left(\VV,\frac{\epsilon (1 - \gamma)}{4N}\right) \exp\left( - \frac{(1 - \gamma)^2}{4N} \epsilon^2\right).
\end{align*}


For a given $\delta$, we thus want to find the smallest $\epsilon$ such that:
\begin{align*}
    \delta &\geq \NN \left(\VV,\frac{\epsilon (1 - \gamma)}{4N}\right) \exp\left( - \frac{(1 - \gamma)^2}{4N} \epsilon^2\right).
\end{align*}
That is:
\begin{align*}
    \log(\frac{1}{\delta}) &\leq \frac{(1 - \gamma)^2}{4N} \epsilon^2 + |\SS| \min(0, \log(\frac{\epsilon (1 - \gamma)}{4N})).
\end{align*}
Picking $\epsilon = O\left(\log(N) \frac{\sqrt{N\log(\frac{1}{\delta})}}{1-\gamma} \right) = \tilde{O}\left( \frac{\sqrt{N\log(\frac{1}{\delta})}}{1-\gamma} \right) $ guarantees that the inequality is verified asymptotically.

Combining these two steps, we have shown overall, with probability at least $1-\delta$,
\begin{align*}
\sum_{n=1}^N   (g_n - \nabla h_n(x_n))^\t y_N^*  =\tilde{O}\left( \frac{\sqrt{NK\log(\frac{1}{\delta})}}{1-\gamma} \right).
\end{align*}

\subsection{Summary}

In the previous subsections, we have provided high probability upper bounds for each term in the decomposition
\begin{align*} 
\regret_N (y_N^*)\leq 
\left(\sum_{n=1}^{N} (\nabla h_n(x_n)-g_n)^\t x_n\right)  
+ \left( \max_{x\in\XX} \sum_{n=1}^N g_n^\t (x_n -x) \right)+ \left(\sum_{n=1}^N   (g_n - \nabla h_n(x_n))^\t y_N^* \right)
\end{align*}
implying with probability at least $1-\delta$, 
\begin{align*}
\regret_N (y_N^*) \leq \tilde{O}\left( \frac{\sqrt{NK\log(\frac{1}{\delta})}}{1-\gamma} \right) +   \tilde{O}\left(\frac{\sqrt{KN}}{1-\gamma} + \frac{\sqrt{K^3 \log\frac{1}{\delta}}}{1-\gamma} \right)  
=  \tilde{O}\left( \frac{\sqrt{N|\SS||\AA|\log(\frac{1}{\delta})}}{1-\gamma} \right) 
\end{align*}
By \cref{th:reduction of RL}, this would imply with probability at least $1-\delta$, 
\begin{align*}
 V^{\hat{\pi}_N}(p)  \geq V^*(p)  -  \frac{\regret_N(y_N^*)}{N} \geq  V^*(p)  -  \tilde{O}\left( \frac{\sqrt{|\SS||\AA|\log(\frac{1}{\delta})}}{(1-\gamma)\sqrt{N}} \right)
\end{align*}
In other words, the sample complexity of mirror descent to obtain an $\epsilon$ approximately optimal policy (i.e. $ V^*(p) - V^{\hat{\pi}_N}(p) \leq \epsilon$) is at most 
$
\tilde{O}\left(\frac{|\SS||\AA|\log(\frac{1}{\delta})}{(1-\gamma)^2 \epsilon^2} \right)
$.

\section{Sample Complexity of Mirror Descent with Basis Functions} \label{app:sample complexity using function approximators}

Here we provide further discussions on the sample complexity of running \cref{alg:md for RL} with linearly parameterized function approximators and the proof of \cref{th:sample complexity of mirror descent with basis}.
\SampleComplexityOfMirrorDescentWithBasis*

%

\subsection{Setup}

We suppose that the decision variable is parameterized in the form $x_\theta = (\Phib\thetab_v, \Psib\thetab_\mu)$, where $\Phib,\Psib$ are given nonlinear basis functions and $(\thetab_v,\thetab_\mu) \in \Theta$ are the parameters to learn.
For modeling the value function, we suppose each column in $\Phib$ is a vector (i.e. function) such that its  $\norm{\cdot}_\infty$ is less than one. For modeling the state-action distribution, we suppose each column in $\Psib$ is a state-action distribution from which we can draw samples. 
This choice implies that every column of $\Phib$ belongs to $\VV$, and every column of $\Psib$ belongs to $\MM$.

Considering the geometry of $\Phib$ and $\Psib$, we consider a compact and convex parameter set
\begin{align*}
\Theta = \{ \theta = (\thetab_v, \thetab_\mu): \norm{\thetab_v}_2 \leq \frac{C_v}{\sqrt{dim(\thetab_v)}}, \thetab_\mu \geq 0, \norm{\thetab_\mu}_1 = 1  \}
\end{align*}

where $C_v < \infty$. 
The constant $C_v$ acts as a regularization in learning: if it is too small, the bias (captured as $\epsilon_{\Theta,N}$ in \cref{cr:reduction for funcapp} restated below) becomes larger; if it is too large, the learning becomes slower.

This choice of $\Theta$ makes sure, for $\theta = (\thetab_v, \thetab_\mu) \in \Theta$,  $\Psib\thetab_\mu \in \MM$ and $\norm{\Phib\thetab_v}_\infty \leq \norm{\thetab_v}_1 \leq C_v$.
Therefore, by the above construction, we can verify that the requirement in \cref{cr:reduction for funcapp} is satisfied, i.e. for $\theta = (\thetab_v, \thetab_\mu) \in \Theta$, we have
$(\Phib\thetab_v, \Psib \thetab_\mu) \in \XX_\Theta
$.
\ReductionForFunctionApporximator*

\subsection{Online Loss and Sampled Gradient}

Let $\theta = (\thetab_v,\thetab_\mu) \in \Theta$.
In view of the parameterization above, we can identify the  online loss in \eqref{eq:modified per-round loss} in the parameter space as 
\begin{align} \label{eq:modified per-round loss (parameter space)}
\textstyle
h_n(\theta) \coloneqq \bb_{\mub_n}^\t \Phib \thetab_v  + \thetab_\mu^\t \Psib^\t  (\frac{1}{1-\gamma} \oneb - \ab_{\vb_n})
\end{align}
where we have the natural identification $x_n = (\vb_n, \mub_n) = (\Phib\thetab_{v,n}, \Psib\thetab_{\mu,n})$ and $\theta_n  = (\thetab_{v,n},\thetab_{\mu,n}) \in \Theta$ is the decision made by the online learner in the $n$th round. Note that because this extension of \cref{alg:md for RL} makes sure $\norm{\thetab_{\mu,n}}_1 = 1$ for every iteration, we can still use $h_n$.
For writing convenience, we will continue to overload $h_n$ as a function of parameter $\theta$ in the following analyses.

Mirror descent requires gradient estimates of $\nabla h_n(\theta_n)$. Here we construct an unbiased stochastic estimate of $\nabla h_n(\theta_n)$ as
\begin{align} \label{eq:stochastic gradient estimate (parameter space)}
\hspace{-2mm} g_n
= \begin{bmatrix}
\gb_{n,v}\\
\gb_{n,\mu}
\end{bmatrix}
= \begin{bmatrix}
\Phib^\t (\tilde{\pb}_n + \frac{1}{1-\gamma}(\gamma\tilde{\Pb}_n -\Eb_n)^\t \tilde{\mub}_n) \\
dim(\thetab_\mu)\hat{\Psib}_n^\t (\frac{1}{1-\gamma
} \hat{\oneb}_n - \hat{\rb}_n - \frac{1}{1-\gamma}(\gamma\hat{\Pb}_n-\hat{\Eb}_n)\vb_n)
\end{bmatrix}
\end{align}
using two calls of the generative model (again we overload the symbol $g_n$ for the analyses in this section):
\begin{itemize}
\item The upper part $\gb_{n,v}$ is constructed similarly as before in \eqref{eq:stochastic gradient estimate}: First we sample the initial state from the initial distribution, 
the state-action pair using the learned state-action distribution, and then the transited-to state at the sampled state-action pair. We evaluate $\Phib$'s values at those samples to construct  $\gb_{n,v}$. Thus,  $\gb_{n,v}$ would generally be a dense vector of size $dim(\thetab_v)$ (unless the columns of $\Phib$ are sparse to begin with).

\item The lower part $\gb_{n,\mu}$ is constructed slightly differently. Recall for the tabular version in \eqref{eq:stochastic gradient estimate}, we uniformly sample over the state and action spaces. Here instead we first sample uniformly a column (i.e. a  basis function) in $\Psib$ and then sample a state-action pair according to the sampled column, which is a distribution by design. Therefore, the multiplier due to uniform sampling in the second row of \eqref{eq:stochastic gradient estimate (parameter space)} is now $dim(\thetab_\mu)$ rather than $|\SS||\AA|$ in \eqref{eq:stochastic gradient estimate}. 
The matrix $\hat{\Psib}_n$ is extremely sparse, where only the single sampled entry (the column and the state-action pair) is one and the others are zero.
In fact, one can think of the tabular version as simply using basis functions $\Psib = \Ib$, i.e. each column is a delta distribution. Under this identification, the expression in \eqref{eq:stochastic gradient estimate (parameter space)} matches the one in \eqref{eq:stochastic gradient estimate}.
\end{itemize}
It is straightforward to verify that $\E[g_n] = \nabla h_n(\theta_n)$ for $g_n$ in \eqref{eq:stochastic gradient estimate (parameter space)}.

\subsection{Proof of \cref{th:sample complexity of mirror descent with basis}}

We follow the same steps of the analysis of the tabular version. We will highlight the differences/improvement due to using function approximations.

First, we use \cref{cr:reduction for funcapp} in place of \cref{th:reduction of RL}. To properly handle the randomness, we revisit its derivation to slightly tighten the statement, which was simplified for the sake of cleaner exposition.
Define 
\begin{align*}
y_{N,\theta}^* = (\vb_{N,\theta}^*, \mub_\theta^*) \coloneqq \argmax_{x_\theta\in\XX_\theta} r_{ep}(\hat{x}_N; x_\theta).
\end{align*}
For writing convenience, let us also denote $\theta_N^*  = (\thetab_{v,N}^*,\thetab_{\mu}^*) \in \Theta$ as the corresponding parameter of $y_{N,\theta}^*$.
We remark that $\mub_\theta^*$ (i.e. $\thetab_{\mu}^*$), which tries to approximate $\mub^*$, is fixed before the learning process, whereas $\vb_{N,\theta}^*$ (i.e. $\thetab_{v,N}^*$) could depend on the stochasticity in the learning. 
Using this new notation and the steps in the proof of \cref{cr:reduction for funcapp}, we can write
\begin{align*}
&V^*(p) - V^{\hat{\pi}_N}(p) 
= r_{ep}(\hat{x}_N;y_N^*)\\
&= \epsilon_{\Theta,N} + r_{ep}(\hat{x}_N;y_{N,\theta}^*) \leq  \epsilon_{\Theta,N} + \frac{\regret_N(y_{N,\theta}^*)}{N}
\end{align*}
where the first equality is \cref{pr:clever residual bound}, the last inequality follows the proof of \cref{th:reduction of RL}, and we recall the definition
$
\epsilon_{\Theta,N} = r_{ep}(\hat{x}_N;y_N^*) - r_{ep}(\hat{x}_N;y_{N,\theta}^*)
$.

The rest of the proof is very similar to that of \cref{th:reduction of RL}, because linear parameterization does not change the convexity of the loss sequence. Let $y_N^* = (\vb^{\hat{\pi}_N}, \mub^*)$. We bound the regret by the following rearrangement. 
\begin{align} \label{eq:regret rearrangement (parameter space)}
\regret_N (y_{N,\theta}^*)
&= \sum_{n=1}^{N} l_n(x_n) - \sum_{n=1}^N l_n(y_{N,\theta}^*) \nonumber \\
&= \sum_{n=1}^{N} h_n(\theta_n) - \sum_{n=1}^N h_n(\theta_N^* ) \nonumber \\
&= \sum_{n=1}^{N} \nabla h_n(\theta_n)^\t (\theta_n-\theta_N^*) \nonumber \\
&= \left(\sum_{n=1}^{N} (\nabla h_n(\theta_n)-g_n)^\t \theta_n\right)  + \left(\sum_{n=1}^N g_n^\t (\theta_n - \theta_N^*) \right) + \left(\sum_{n=1}^N   (g_n - \nabla h_n(\theta_n))^\t \theta_N^* \right) \nonumber\\
&\leq 
\left(\sum_{n=1}^{N} (\nabla h_n(\theta_n)-g_n)^\t \theta_n\right)  
+ \left( \max_{\theta\in\Theta} \sum_{n=1}^N g_n^\t (\theta_n -\theta) \right)+ \left(\sum_{n=1}^N   (g_n - \nabla h_n(\theta_n))^\t \theta_N^* \right)
\end{align}
where the second equality is due to the identifcation in \eqref{eq:modified per-round loss (parameter space)}.

We will solve this online learning problem with mirror descent 
\begin{align} \label{eq:mirror descent (parameter space)}
\theta_{n+1} = \argmin_{\theta\in\Theta} \lr{g_n}{\theta} + \frac{1}{\eta} B_{R}(\theta||\theta_n)
\end{align}
with step size $\eta>0$ and  a Bregman divergence that is a straightforward extension of \eqref{eq:Bregman divergence choice}
\begin{align} \label{eq:Bregman divergence choice (paramter space)}
\textstyle
B_R(\theta'||\theta) =  \frac{1}{2} \frac{dim(\theta_v)}{C_v^2} \norm{\thetab_v' - \thetab_v}_2^2 + KL(\thetab_\mu'||\thetab_\mu) 
\end{align}
where the constant $\frac{dim(\theta_v)}{C_v^2}$ is chosen to make the size of Bregman divergence dimension-free (at least up to $\log$ factors).
Below we analyze the size of the three terms in \eqref{eq:regret rearrangement (parameter space)} like what we did for \cref{th:sample complexity of mirror descent}.

\subsection{The First Term: Martingale Concentration} 

The first term is a martingale. We will use this part to highlight the different properties due to using basis functions. 
The proof follows the steps in \cref{app:martingale concentration}, but now the martingale difference of interest is instead 
\begin{align*}
M_n - M_{n-1} 
&= (\nabla h_n(\theta_n)-g_n)^\t \theta_n \\
&= ( \Psib^\t(\kappa \oneb - \ab_{\vb_n}) - \gb_{n,\mu})^\t \thetab_{\mu,n} 
+ (\Phib^\t \bb_{\mub_n} -\gb_{n,v})^\t \thetab_{v,n}
\end{align*}
They now have nicer properties due to the way $\gb_{n,\mu}$ is sampled.

For the first term $(\Psib^\t(\kappa \oneb - \ab_{\vb_n}) - \gb_{n,\mu})^\t \thetab_{\mu,n}$, we use the lemma below, where we recall the filtration $F_n$ is naturally defined as $\{\theta_1, \dots, \theta_n\}$.
\begin{lemma} \label{lm:martingale difference for mu (parameter space)}
	Let $\theta = (\thetab_v, \thetab_\mu) \in \Theta$ be arbitrary that is chosen independent of the randomness of $\gb_{n,\mu}$ when $F_{n-1}$ is given. Then it holds
	$ |(\kappa \oneb - \ab_{\vb_n} - \gb_{n,\mu})^\t \thetab|  \leq \frac{2(1+dim(\thetab_\mub))}{1-\gamma}$ 
	and $\Vbb_{|F_{n-1}}[(\kappa \oneb - \ab_{\vb_n} - \gb_{n,\mu})^\t \thetab_n] \leq \frac{4dim(\thetab_\mub)}{(1-\gamma)^2}$.
\end{lemma}
\begin{proof}
	By triangular inequality, 
	\begin{align*}
	|(\Psib^\t(\kappa \oneb - \ab_{\vb_n}) - \gb_{n,\mu})^\t \thetab_\mu|  \leq  |(\kappa \oneb - \ab_{\vb_n})^\t \Psib \thetab_\mu |+ | \gb_{n,\mu}^\t \thetab_\mu| 
	\end{align*}
	For the deterministic part, using \cref{lm:size of a and b} and H\"older's inequality,
	\begin{align*}
	|(\kappa \oneb - \ab_{\vb_n})^\t \Psib \thetab_\mu| 
	\leq \kappa + \norm{\ab_{\vb_n}}_\infty\norm{ \Psib \thetab_\mu}_1 \leq \frac{2}{1-\gamma}
	\end{align*}
	For the stochastic part, let $k_n$ denote the sampled column index, 
	$i_n$ be index of the sampled state-action pair using the column of $k_n$, and $j_n$ be the index of the transited state sampled at the pair given by $i_n$. With abuse of notation, we will use $i_n$ to index both $\SS\times\AA$ and $\SS$. 
	Let $\mub = \Psib \thetab_\mu$.	With this notation, we may derive
	\begin{align*}
	|\gb_{n,\mu}^\t \thetab_\mu |
	&= |dim(\thetab_\mu) \thetab_\mub^\t \hat{\Psib}_n^\t (\kappa \hat{\oneb}_n - \hat{\rb}_n - \frac{1}{1-\gamma}(\gamma\hat{\Pb}_n-\hat{\Eb}_n)\vb_n)|\\
	&=dim(\thetab_\mu) \theta_{\mu, k_n}|\kappa  - r_{i_n} - \frac{\gamma v_{n,j_n} - v_{n,i_n}}{1-\gamma} |\\
	&\leq \frac{2 dim(\thetab_\mu) \theta_{\mu, k_n}}{1-\gamma} \leq  \frac{2 dim(\thetab_\mu)}{1-\gamma}
	\end{align*}
	where we use the facts $r_{i_n},  v_{n,j_n}, v_{n,i_n} \in [0,1]$ and $\theta_{\mu, k_n} \leq 1$.

	Let $\psib_{\mu}^{(k)} $ denote the $k$th column of $\Psib$.
	For $\Vbb_{|F_{n-1}}[(\kappa \oneb - \ab_{\vb_n} - \gb_{n,\mu})^\t \thetab_n]$, we can write it as 
	\begin{align*}
	\Vbb_{|F_{n-1}}[(\Psib^\t(\kappa \oneb - \ab_{\vb_n}) - \gb_{n,\mu})^\t \thetab_\mu] 
	&= \Vbb_{|F_{n-1}}[ \gb_{n,\mu}^\t \thetab_n] \\
	&\leq \E_{|F_{n-1}}[ |\gb_{n,\mu}^\t \thetab_n|^2] \\
	&= \sum_{k_n} \frac{1}{dim(\thetab_\mu)} \sum_{i_n} \psi_{\mu,i_n}^{(k_n)} \E_{j_n|i_n} \left[ dim(\thetab_\mu)^2 \theta_{\mu, k_n}^2 \left(\kappa  - r_{i_n} - \frac{\gamma v_{n,j_n} - v_{n,i_n}}{1-\gamma}\right)^2 \right]\\
	&\leq  \frac{4 dim(\thetab_\mu) }{(1-\gamma)^2} \sum_{k_n} \theta_{\mu,k_n}^2 \sum_{i_n} \psi_{\mu,i_n}^{(k_n)} \\
	&\leq  \frac{4 dim(\thetab_\mu) }{(1-\gamma)^2} \left( \sum_{k_n} \theta_{\mu,k_n}\right)^2   \leq \frac{4 dim(\thetab_\mu) }{(1-\gamma)^2}
	\end{align*}
	where in the second inequality we use the fact that $|\kappa  - r_{i_n} - \frac{\gamma v_{n,j_n} - v_{n,i_n}}{1-\gamma} | \leq \frac{2}{1-\gamma}$.
\end{proof}

For the second term $ (\Phib^\t \bb_{\mub_n} -\gb_{n,v})^\t \thetab_{v,n}$, we use this lemma.
\begin{lemma} \label{lm:martingale difference for v (paramter space)}
	Let $\thetab \in \VV$ be arbitrary that is chosen independent of the randomness of $\gb_{n,v}$ when $F_{n-1}$ is given. Then it holds
	$ |(\Phib^\t \bb_{\mub_n} -\gb_{n,v})^\t \thetab|  \leq \frac{4C_v}{1-\gamma}$ 
	and $\Vbb_{|F_{n-1}}[(\Phib^\t \bb_{\mub_n} -\gb_{n,v})^\t \thetab] \leq \frac{4C_v^2}{(1-\gamma)^2}$.
\end{lemma}
\begin{proof}
	We appeal to \cref{lm:size of a and b}, which shows $\norm{\bb_{\mub_n}}_1 \leq \frac{2}{1-\gamma}$ and
	\begin{align*}
	\norm{\tilde{\pb}_n + \frac{1}{1-\gamma}(\gamma\tilde{\Pb}_n -\Eb_n)^\t \tilde{\mub}_n}_1 \leq \frac{2}{1-\gamma}
	\end{align*}
	Therefore, overall we can derive
	\begin{align*}
	|(\Phib^\t \bb_{\mub_n} -\gb_{n,v})^\t \thetab| &\leq \left( \norm{\bb_{\mub_n}}_1\ + \norm{\tilde{\pb}_n + \frac{1}{1-\gamma}(\gamma\tilde{\Pb}_n -\Eb_n)^\t \tilde{\mub}_n}_1  \right) \norm{\Phib\thetab_v}_\infty \leq \frac{4 C_v}{1-\gamma}
	\end{align*}
	where we use again each column in $\Phib$ has $\norm{\cdot}_\infty$ less than one, and $\norm{\cdot}_\infty \leq \norm{\cdot}_2$.
	Similarly, for the variance, we can write
	\begin{align*}
	\Vbb_{|F_{n-1}}[(\Phib^\t \bb_{\mub_n} -\gb_{n,v})^\t \thetab]
	&= \Vbb_{|F_{n-1}}[\gb_{n,v}^\t \thetab]\leq \Ebb_{|F_{n-1}}[(\gb_{n,v}^\t \thetab)^2] \leq \frac{4C_v^2}{(1-\gamma)^2} \qedhere
	\end{align*}
\end{proof}

From the above two lemmas, we see the main difference from the what we had in \cref{app:martingale concentration} for the tabular case is that, the martingale difference now scales in $O\left(\frac{C_v + dim(\thetab_\mub)}{1-\gamma}\right)$ instead of $O\left(\frac{|\SS||\AA|}{1-\gamma}\right)$, and its variance scales in $O\left(\frac{C_v^2 + dim(\thetab_\mub)}{(1-\gamma)^2}\right)$ instead of $O\left(\frac{|\SS||\AA|}{(1-\gamma)^2}\right)$. We note the constant $C_v$ is universal, independent of the problem size.

Following the similar steps in \cref{app:martingale concentration}, these new results imply that
\begin{align*}
P\left( \sum_{n=1}^{N} (\nabla h_n(\theta_n)-g_n)^\t \theta_n > \epsilon \right) \leq \exp\left(\frac{- \epsilon^2}{2 N \sigma^2 (1+\frac{b\epsilon}{3N\sigma^2})}   \right)
\end{align*}
with $b = O\left(\frac{C_v + dim(\thetab_\mub)}{1-\gamma}\right)$ and $O\left(\frac{C_v^2 + dim(\thetab_\mub)}{(1-\gamma)^2}\right)$. This implies that, with probability at least $1-\delta$, it hold
\begin{align*}
\sum_{n=1}^{N} (\nabla h_n(\theta_n)-g_n)^\t \theta_n
= \tilde{O}\left( \frac{\sqrt{N( C_v^2+dim(\thetab_\mub))\log(\frac{1}{\delta})}}{1-\gamma} \right)
\end{align*}

\subsection{Static Regret of Mirror Descent} 

Again the steps here are very similar to those in \cref{app:static regret of mirror descent}.
We concern bounding the static regret.
\begin{align*}
\max_{\theta\in\Theta} \sum_{n=1}^N g_n^\t (\theta_n -\theta)
\end{align*}
From \cref{app:static regret of mirror descent}, we recall this can be achieved by the mirror descent's optimality condition. The below inequality is true, for any $\theta'\in\Theta$:
\begin{align*}
\sum_{n=1}^{N} \lr{g_n}{\theta_{n} - \theta'} &\leq \frac{1}{\eta} B_R(\theta'||\theta_1) + \sum_{n=1}^{N} \lr{g_n}{\theta_{n+1} - \theta_n} - \frac{1}{\eta}  B_R(\theta_{n+1}||\theta_n)  
\end{align*} 
Based on our choice of Bregman divergence given in \eqref{eq:Bregman divergence choice (paramter space)}, i.e.
\begin{align} \tag{\ref{eq:Bregman divergence choice (paramter space)}}
\textstyle
B_R(\theta'||\theta) =  \frac{1}{2} \frac{dim(\theta_v)}{C_v^2} \norm{\thetab_v' - \thetab_v}_2^2 + KL(\thetab_\mu'||\thetab_\mu),
\end{align}
we have
$
\frac{1}{\eta}  B_R(\theta'||\theta_1)\leq \frac{\tilde{O}(1)}{\eta}
$.
For each $\lr{g_n}{\theta_{n+1} - \theta_n} - \frac{1}{\eta}  B_R(\theta_{n+1}||\theta_n)$, we will use again the two basic lemmas we proved in \cref{app:static regret of mirror descent}.
\GradientDescentBound*
\EXPBound*
Thus, we have the upper bound 
\begin{align*}
\lr{g_n}{\theta_{n+1} - \theta_n} - \frac{1}{\eta}  B_R(\theta_{n+1}||\theta_n)
= \frac{C_v^2}{dim(\theta_v)} \frac{\eta \norm{\gb_{n,v}}_2^2}{2} + \frac{\eta \norm{\gb_{n,\mu}}_{\thetab_{\mu,n}}^2}{2}
\end{align*}
Together with the upper bound on $\frac{1}{\eta}  B_R(x'||x_1)$, it implies that
\begin{align} \label{eq:initial telesum of md}
\sum_{n=1}^{N} \lr{g_n}{x_{n} - x'} &\leq  \frac{1}{\eta}  B_R(x'||x_1) + \sum_{n=1}^{N} \lr{g_n}{x_{n+1} - x_n} - \frac{1}{\eta} B_R(x_{n+1}||x_n)  \nonumber \\
&\leq \frac{\tilde{O}(1)}{\eta} + \frac{\eta}{2} \sum_{n=1}^{N} \frac{C_v^2}{dim(\theta_v)}  \norm{\gb_{n,v}}_2^2 + \norm{\gb_{n,\mu}}_{\thetab_{\mu,n}}^2
\end{align}

We can expect, with high probability, $\sum_{n=1}^{N} \frac{C_v^2}{dim(\theta_v)} \norm{\gb_{n,v}}_2^2 + \norm{\gb_{n,\mu}}_{\thetab_{\mu,n}}^2$ concentrates toward its expectation, i.e.
\begin{align*}
\sum_{n=1}^{N} \frac{C_v^2}{dim(\theta_v)} \norm{\gb_{n,v}}_2^2 + \norm{\gb_{n,\mu}}_{\thetab_{\mu,n}}^2 \leq \sum_{n=1}^{N} \E\left[\frac{C_v^2}{dim(\theta_v)}\norm{\gb_{n,v}}_2^2 + \norm{\gb_{n,\mu}}_{\thetab_{\mu,n}}^2\right] + o(N)
\end{align*}
To bound the right-hand side, we will use the upper bounds below, which largely follow the proof of \cref{lm:martingale difference for mu (parameter space)} and \cref{lm:martingale difference for v (paramter space)}.
\begin{lemma} \label{lm:size of sampled grad (parameter space)}
	$
	\E[\norm{\gb_{n,v}}_2^2] \leq \frac{4dim(\thetab_v)}{(1-\gamma)^2} $ and $\E[\norm{\gb_{n,\mu}}_{\thetab_{\mu,n}}^2] \leq \frac{4 dim(\thetab_\mu)}{(1-\gamma)^2}$.
\end{lemma}
\begin{lemma}
	$\norm{\gb_{n,v}}_2^2 \leq \frac{4 dim(\thetab_v)}{(1-\gamma)^2}$ and $\norm{\gb_{n,\mu}}_{\thetab_{\mu,n}}^2 \leq \frac{4 dim(\thetab_\mu)^2}{(1-\gamma)^2}$.
\end{lemma}
By Azuma-Hoeffding's inequality in \cref{lm:martgingale (Hoeffding)},
\begin{align*}
\sum_{n=1}^{N} \frac{C_v^2}{dim(\theta_v)}  \norm{\gb_{n,v}}_2^2 + \norm{\gb_{n,\mu}}_{\thetab_{\mub,n}}^2 
&\leq   \sum_{n=1}^{N} \E\left[ \frac{C_v^2}{dim(\theta_v)} \norm{\gb_{n,v}}_2^2 +  \norm{\gb_{n,\mu}}_{\thetab_{\mub,n}}^2 \right] +
O\left(\frac{C_v^2+dim(\thetab_\mu)^2}{(1-\gamma)^2} \sqrt{N \log\left(\frac{1}{\delta}\right)}  \right) \\
&\leq  O\left( \frac{C_v^2+dim(\thetab_\mu)}{(1-\gamma)^2} N \right)+ O\left(\frac{C_v^2+dim(\thetab_\mu)^2}{(1-\gamma)^2} \sqrt{N \log\left(\frac{1}{\delta}\right)}  \right)
\end{align*}

Now we suppose we set $\eta = O\left(  \frac{1-\gamma}{\sqrt{N(C_v^2+dim(\thetab_\mu))}} \right)$. We have
\begin{align*}
\sum_{n=1}^{N} \lr{g_n}{\theta_{n} - \theta'} 
&\leq 
\frac{\tilde{O}(1)}{\eta} + \frac{\eta}{2} \sum_{n=1}^{N} \frac{C_v^2}{dim(\theta_v)}  \norm{\gb_{n,v}}_2^2 + \norm{\gb_{n,\mu}}_{\thetab_{\mu,n}}^2
 \leq  \tilde{O}\left(\frac{\sqrt{(C_v^2+dim(\thetab_\mu))N}}{1-\gamma}  \right)  
\end{align*}

\subsubsection{Union Bound}
Lastly we use an union bound to handle the term 
\begin{align*}
\sum_{n=1}^N   (g_n - \nabla h_n(\theta_n))^\t \theta_N^* 
\end{align*}
We follow the steps in \cref{app:union bound}: we will use again the fact that $\theta_N^*  = (\thetab_{v,N}^*,\thetab_{\mu}^*) \in \Theta$, so we can handle the part with $\thetab_{\mu}^*$ using the standard martingale concentration, and the part with $\thetab_{v,N}^*$ using the union bound.

Using the previous analyses, we see can first show that the martingale due to the part $\thetab_{\mu}^*$ concentrates in 
$
\tilde{O}\left( \frac{\sqrt{Ndim(\thetab_\mub) \log(\frac{1}{\delta})}}{1-\gamma} \right)
$.
Likewise, using the union bound, we can show 
 the martingale due to the part $\thetab_{v,N}^*$ concentrates in 
$
 \tilde{O}\left( \frac{\sqrt{N C_v^2\log(\frac{\NN }{\delta})}}{1-\gamma} \right)
$
where $\NN$ some proper the covering number of the set $\left\{\thetab_v :  \norm{\thetab_v}_2 \leq \frac{C_v}{\sqrt{dim(\thetab_v)}} \right\}$. 
Because $\log\NN = O(\dim(\thetab_v))$ for an Euclidean ball. We can combine the two bounds and show together
\begin{align*}
\sum_{n=1}^N   (g_n - \nabla h_n(\theta_n))^\t \theta_N^* = \tilde{O}\left( \frac{\sqrt{N (C_v^2\dim(\thetab_v) +dim(\thetab_\mub)) \log(\frac{1}{\delta})}}{1-\gamma} \right)
\end{align*}

\subsubsection{Summary}
Combining the results of the three parts above, we have, with probability $1-\delta$,
\begin{align*}
&\regret_N (y_{N,\theta}^*)\\
&\leq 
\left(\sum_{n=1}^{N} (\nabla h_n(\theta_n)-g_n)^\t \theta_n\right)  
+ \left( \max_{\theta\in\Theta} \sum_{n=1}^N g_n^\t (\theta_n -\theta) \right)+ \left(\sum_{n=1}^N   (g_n - \nabla h_n(\theta_n))^\t \theta_N^* \right)\\
&=  \tilde{O}\left( \frac{\sqrt{N(dim(\thetab_\mub) + C_v^2)\log(\frac{1}{\delta})}}{1-\gamma} \right) + \tilde{O}\left(\frac{\sqrt{(C_v^2+dim(\thetab_\mu))N}}{1-\gamma}  \right)  +  \tilde{O}\left( \frac{\sqrt{N (C_v^2\dim(\thetab_v) +dim(\thetab_\mub)) \log(\frac{1}{\delta})}}{1-\gamma} \right)\\
&=  \tilde{O}\left( \frac{\sqrt{N dim(\Theta) \log(\frac{1}{\delta})}}{1-\gamma} \right)
\end{align*}
where the last step is due to $C_v$ is a universal constant. Or equivalently, the above bounds means a sample complexity in $\tilde{O}\left( \frac{ dim(\Theta) \log(\frac{1}{\delta})}{(1-\gamma)^2 \epsilon^2} \right)$. Finally, we recall the policy performance has a bias $\epsilon_{\Theta,N}$ in \cref{cr:reduction for funcapp} due to using function approximators. Considering this effect, we have the final statement.
\clearpage

\end{document}